%% file: bdd-mv_arxiv.tex
\documentclass[11pt,oneside]{article}

\usepackage{./config/all}

\input{./config/bib}

\input{./config/mathfonts}

\input{./config/symbols}

\usepackage[pdfauthor={Matthew J. Holland},%
colorlinks=true,
linkcolor=blue,%
citecolor=blue]{hyperref}
\hypersetup{pdftitle={Robust variance-regularized risk minimization with concomitant scaling}} 

\begin{document}

\title{\textbf{Robust variance-regularized risk minimization\\with concomitant scaling}}
\author{
  Matthew J.~Holland\\
  Osaka University
}
\date{} 

\maketitle

\begin{abstract}
Under losses which are potentially heavy-tailed, we consider the task of minimizing sums of the loss mean and standard deviation,  without trying to accurately estimate the variance. By modifying a technique for variance-free robust mean estimation to fit our problem setting, we derive a simple learning procedure which can be easily combined with standard gradient-based solvers to be used in traditional machine learning workflows. Empirically, we verify that our proposed approach, despite its simplicity, performs as well or better than even the best-performing candidates derived from alternative criteria such as CVaR or DRO risks on a variety of datasets.
\end{abstract}

\tableofcontents

\section{Introduction}\label{sec:intro}

Traditionally, the ``textbook definition'' of a statistical machine learning problem is formulated in terms of making decisions which minimize the expected value of a random loss \citep{devroye1996ProbPR,mohri2012Foundations,vapnik1999NSLT}. More precisely, the traditional setup has us minimize $\exx_{\ddist}\loss(h)$ with respect to a decision $h$, where we denote random losses as $\loss(h) \defeq \ell(h;\rdv{Z})$, with a random data point $\rdv{Z} \sim \ddist$, and $\ell(\cdot)$ is a loss function assigning real values to $(\text{decision},\text{data})$ pairs. This problem class is very general in that it covers a wide range of learning problems both supervised and unsupervised, but it is limited in the sense that it only aspired to be optimal \emph{on average}, with no guarantees for other aspects of performance such as loss deviations, resilience to worst-case examples and distribution shift, sub-population disparity, and class-balanced error. While it is sometimes possible to account for these issues by modifying the base loss function $\ell$ (e.g., logit-adjusted softmax cross-entropy for balanced error \citep{menon2021a}), there is a growing literature looking at principled, systematic modifications to the ``risk,'' i.e., a non-random numerical property of the \emph{distribution} of $\loss(h)$ to be optimized in $h$, leaving the base loss $\ell(\cdot)$ fixed. Some prominent examples are weighted sums of loss quantiles \citep{maurer2021a}, distributionally robust optimization (DRO) risk \citep{hashimoto2018a}, conditional value-at-risk (CVaR) \citep{curi2020a}, tilted risk \citep{li2021a}, and more general optimized certainty equivalent (OCE) risks \citep{lee2020a}, among others. It is well-known that many risks can be expressed in terms of location-deviation sums, with the canonical example being a weighted sum of the loss mean and standard deviation (or variance) \citep[\S{2}]{rockafellar2013a}. We refer the reader to some recent surveys \citep{holland2023survey,hu2022a,royset2022arxiv} for more general background on developments in learning criteria.

In this work, the criterion of interest is the mean loss regularized by standard deviation (SD), when losses are allowed to be heavy-tailed. More formally, we allow for heavy tails in the sense that all we assume is that the second moment $\exx_{\ddist}\abs{\loss(h)}^{2}$ is finite, and the ultimate objective of interest is the \term{mean-SD} criterion
\begin{align}\label{eqn:mean_stdev_defn}
\msd_{\ddist}(h;\lambda) \defeq \exx_{\ddist}\loss(h) + \sqrt{\lambda\vaa_{\ddist}\loss(h)}
\end{align}
with loss variance denoted by $\vaa_{\ddist}\loss \defeq \exx_{\ddist}(\loss - \exx_{\ddist}\loss)^{2}$, and weighting parameter $\lambda \geq 0$. This mean-SD objective (\ref{eqn:mean_stdev_defn}) and its mean-variance counterpart have a long history in the literature on decision making under uncertainty, including the influential work of \citet{markowitz1952a} on optimal portfolio selection. In the context of machine learning, it is well-known that one can obtain ``fast rate'' bounds on the expected loss when variance is small (see \citep[\S{1}]{duchi2019a}), though the problem of actually ensuring that loss deviations are sufficiently small is an entirely separate matter. In this direction, \citet{maurer2009a} bound the (population) expected loss using a weighted sum of the \emph{sample} mean and standard deviation. Their ``sample variance penalized'' objective is convenient to compute and can be used to guarantee fast rates in theory, but a lack of convexity makes it hard to minimize in practice. A convex approximation is developed by \citet{duchi2019a}, who show that a sub-class of (empirical) DRO risks can be used to approximate the sample mean-SD objective, again yielding fast rates when the (population) variance is small enough. The critical limitation to this approach is poor guarantees under heavy-tailed losses; while we gain in terms of convexity, the empirical DRO risk of \citep{duchi2019a} is at least as sensitive to outliers as the naive empirical objective (i.e., directly minimizing the sample mean and SD), which is already known to result in highly sub-optimal performance guarantees under heavy tails \citep{brownlees2015a,devroye2016a,hsu2016a}. Recent work by \citet{zhai2021a} studies a natural strategy for robustifying the DRO objective (called DORO), which discards a specified fraction of the largest losses. While the impact of outliers can be reduced under the right setting of DORO, their approach is limited to non-negative losses, and the impact that such one-sided trimming has on the resulting mean-SD sum, our ultimate object of interest, is unknown.

With this context in mind, in this paper we propose a new approach to robustly minimize the objective (\ref{eqn:mean_stdev_defn}) under heavy-tailed losses, without \textit{a priori} knowledge of anything but the fact that variance is finite. Our key technique is based on extending a convex program of \citet{sun2021arxiv} from one-dimensional mean estimation to our mean-variance objective $\msd_{\ddist}(h;\lambda)$ under general losses. After some motivating background points in \S{\ref{sec:background_mean_estimation}}--\S{\ref{sec:background_sun}}, we describe our basic approach and summarize our contributions in \S{\ref{sec:background_bridge}}--\S{\ref{sec:background_ours}}. Theoretical analysis comes in \S{\ref{sec:theory}}, and based upon formal properties of the proposed objective function, we derive a general-purpose procedure summarized in Algorithm \ref{algo:ours}, and tested empirically in \S{\ref{sec:empirical}}. Our main finding is that the simple algorithm we derive works remarkably well on both simulated and real-world datasets without any fine-tuning, despite sacrificing the convexity enjoyed by procedures based on criteria such as CVaR and DRO. All detailed proofs are relegated to the appendix. Software and notebooks to reproduce all results in this paper are provided in an online repository.\footnote{\url{https://github.com/feedbackward/bdd-mv}}

\section{Background}\label{sec:background}

Before we describe our proposed approach to the mean-SD task described in \S{\ref{sec:intro}}, we start with a much simpler problem, namely the task of robust mean estimation. This will allow us to highlight key technical points from the literature which provide both conceptual and technical context for our proposal. Key points from the existing literature are introduced in \S{\ref{sec:background_mean_estimation}}--\S{\ref{sec:background_sun}}, and building upon this we introduce our method in \S{\ref{sec:background_bridge}}--\S{\ref{sec:background_ours}}.

\subsection{Robust mean estimation}\label{sec:background_mean_estimation}

Let $\rdv{X}$ be a random variable. For the moment, our goal will be to construct an accurate empirical estimate of the mean $\exx_{\ddist}\rdv{X}$, assuming only that the variance $\vaa_{\ddist}\rdv{X} = \exx_{\ddist}\rdv{X}^{2}-(\exx_{\ddist}\rdv{X})^{2}$ is both defined and finite. We assume access to an independent and identically distributed (IID) sample $\rdv{X}_{1},\ldots,\rdv{X}_{n}$. Since higher-order moments may be infinite, the tails of $\rdv{X}$ may be ``heavy'' and decidedly non-Gaussian, causing problems for the usual empirical mean. This problem setting is now very well-understood; see \citet{lugosi2019b} for an authoritative reference. One very well-known approach is to use M-estimators \citep{huber2009a}, namely to design an estimator $\rdv{A}_{n} \approx \exx_{\ddist}\rdv{X}$ satisfying
\begin{align}\label{eqn:motivation_m_estimator}
\rdv{A}_{n} \in \argmin_{a \in \RR} \frac{b}{n} \sum_{i=1}^{n}\rho\left(\frac{\rdv{X}_{i}-a}{b}\right)
\end{align}
where $\rho \colon \RR \to \RR_{+}$ is a function that is approximately quadratic near zero, but grows more slowly in the limit, i.e., large deviations are penalized in a sub-quadratic manner, where ``large'' is relative to the scaling parameter $b > 0$, used to control bias. When $\rho(\cdot)$ is convex, differentiable, and the solution set is non-empty, the condition (\ref{eqn:motivation_m_estimator}) is equivalent to
\begin{align}\label{eqn:motivation_m_estimator_stationary}
\frac{1}{n} \sum_{i=1}^{n}\rho^{\prime}\left(\frac{\rdv{X}_{i}-\rdv{A}_{n}}{b}\right) = 0
\end{align}
and when the derivative $\rho^{\prime}(\cdot)$ is bounded on $\RR$ such that 
\begin{align}\label{eqn:motivation_catoni_condition}
-\log(1 - x + \gamma x^{2}) \leq \rho^{\prime}(x) \leq \log(1 + x + \gamma x^{2})
\end{align}
for some constant $0 < \gamma < \infty$, then the approach of \citet{catoni2012a} tells us that when $b^{2}$ scales with $\vaa_{\ddist}\rdv{X}/n$, the deviations $\abs{\rdv{A}_{n}-\exx_{\ddist}\rdv{X}}$ enjoy sub-Gaussian tails, namely upper bounds of the order $\bigO(\sqrt{\log(1/\delta)\vaa_{\ddist}\rdv{X}/n})$ with probability at least $1-\delta$. Under these weak assumptions, such guarantees are essentially optimal \citep{devroye2016a}. While an important result, in practice the need for knowledge of $\vaa_{\ddist}\rdv{X}$ is a significant limitation, since without finite higher-order moments, it is not plausible to obtain variance estimates with such sub-Gaussian guarantees. There do exist other robust estimators such as median-of-means \citep[\S{2.1}]{lugosi2019b} which do not require variance information, and this illustrates the fact that knowledge of the variance is sufficient, although \emph{not} necessary, for sub-Gaussian mean estimation under heavy tails.

\begin{figure}[t]
\centering
\includegraphics[width=1.0\textwidth]{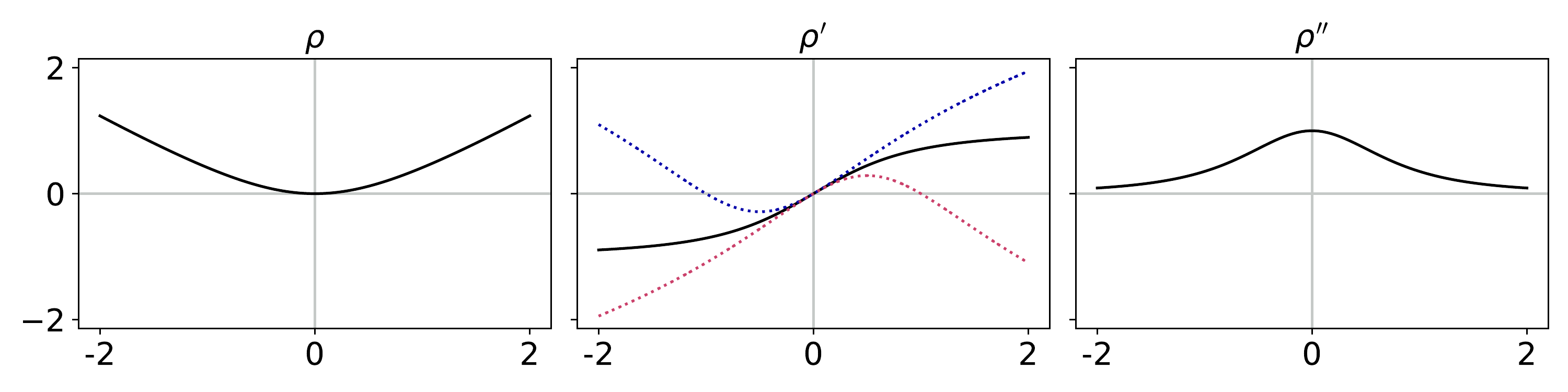}
\caption{From left to right, we plot the graphs of $\rho(\cdot)$, $\rho^{\prime}(\cdot)$, and $\rho^{\prime\prime}(\cdot)$ with $\rho$ as in (\ref{eqn:sun_huber_fn}). In the middle plot, the dotted curves represent the upper (blue) and lower (dark pink) bounds in (\ref{eqn:motivation_catoni_condition}) with $\gamma=1$.}
\label{fig:huber_disp}
\end{figure}

\subsection{Good-enough ancillary scaling}\label{sec:background_sun}

Since sub-Gaussian estimates of the variance $\vaa_{\ddist}\rdv{X}$ are not possible under our weak assumptions, it is natural to ask whether there exists a middle-ground, namely whether or not it is possible to construct a (data-driven) procedure for setting the scale $b > 0$ in (\ref{eqn:motivation_m_estimator}) which is ``good enough'' in the sense that the resulting $\rdv{A}_{n}$ is sub-Gaussian, even though the scale itself cannot be. An initial (affirmative) answer to this question was given in recent work by \citet{sun2021arxiv}, whose basic idea we briefly review here, with some slight re-formulation for readability and additional generality.

Essentially, the underlying idea in \citep{sun2021arxiv} is to utilize the convexity of $\rho$ in (\ref{eqn:motivation_m_estimator}), and to solve for both $a \in \RR$ and $b > 0$ simultaneously, while penalizing $b$ in such a way as to encourage scaling which is ``good enough'' as mentioned. More precisely, the empirical objective
\begin{align}\label{eqn:motivation_sun_obj}
\widehat{\rdv{S}}_{n}(a,b) \defeq \beta b + \frac{b}{n} \sum_{i=1}^{n}\rho\left(\frac{\rdv{X}_{i}-a}{b}\right)
\end{align}
plays a central role, where $0 < \beta < 1$ is a parameter we can control, and $\rho$ is fixed as
\begin{align}\label{eqn:sun_huber_fn}
\rho(x) = \sqrt{x^{2}+1} - 1, \quad x \in \RR
\end{align}
which is differentiable, and satisfies the Catoni condition (\ref{eqn:motivation_catoni_condition}) with $\gamma = 1$ (see Figure \ref{fig:huber_disp}). If we fix $b > 0$, then the solution sets (in $a$) of both $\widehat{\rdv{S}}_{n}(a,b)$ and $b \times \widehat{\rdv{S}}_{n}(a,b)$ are identical, and it should be noted that the re-scaled map $x \mapsto b^{2}\rho(x/b) = b\sqrt{x^{2}+b^{2}}-b^{2}$ closely approximates $x \mapsto x^{2}/2$ as $b$ grows large (Figure \ref{fig:pseudohuber}), and is well-known as the ``pseudo Huber'' or ``smooth Huber'' function, where $b$ acts as a smoothing parameter.\footnote{\citet[\S{1}]{barron2019a} gives a summary of this and related functions from the perspective of loss function design. This is not the only smoothed variant of the classic Huber function \citep{huber1964a}, see for example \citet[\S{6.4.4}]{rey1983Robust}.}

\begin{figure}[t]
\centering
\includegraphics[width=0.4\textwidth]{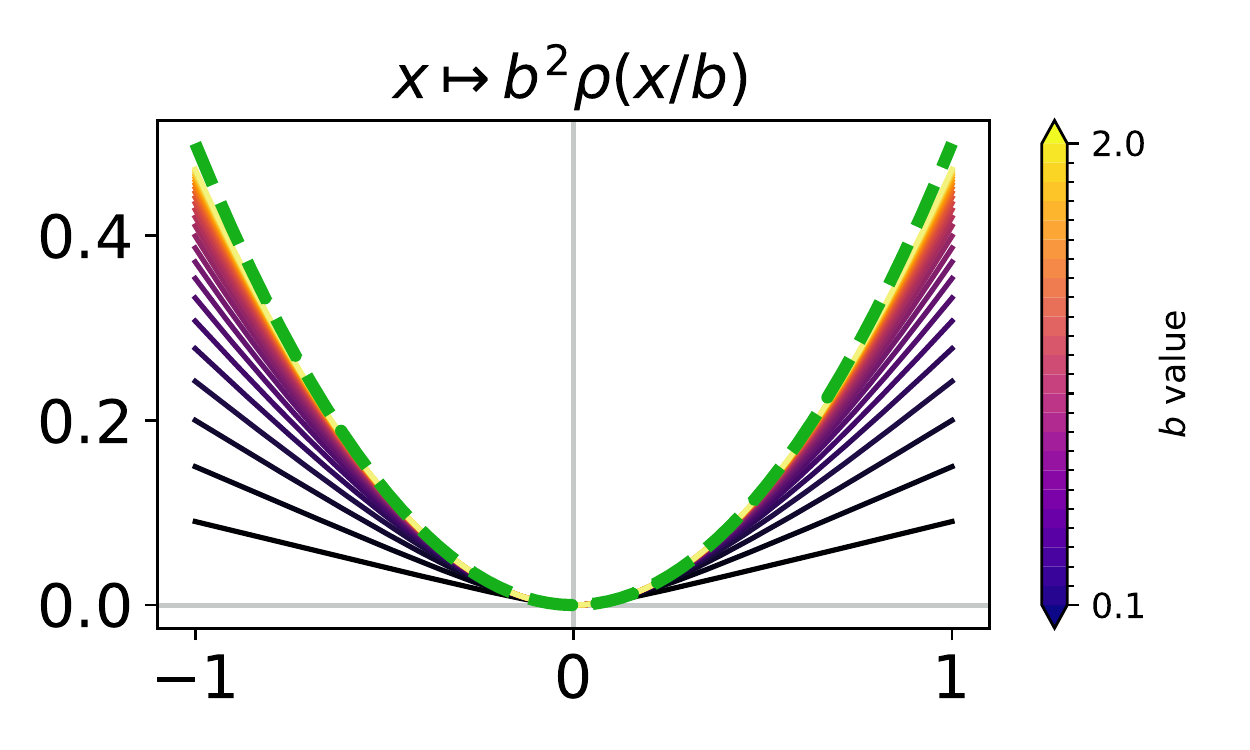}
\caption{Graphs of the smooth Huber function, with $\rho$ as in (\ref{eqn:sun_huber_fn}), over a range of smoothing parameters. For visual comparison, the graph of $x \mapsto x^{2}/2$ is plotted with a thick dashed green curve.}
\label{fig:pseudohuber}
\end{figure}

When considering the joint objective $\widehat{\rdv{S}}_{n}(a,b)$, from the computational side, one important fact is that this function is convex on $\RR \times (0,\infty)$ (see \S{\ref{sec:convexity_smoothness}}). From the statistical side of things, the solutions
\begin{align}
(\rdv{A}_{n},\rdv{B}_{n}) \in \argmin_{a \in \RR, b > 0} \widehat{\rdv{S}}_{n}(a,b)
\end{align}
are such that under certain regularity conditions, the deviations $\abs{\rdv{A}_{n}-\exx_{\ddist}\rdv{X}}$ are nearly optimal (sub-Gaussian, up to poly-logarithmic factors) \citep[\S{3.3}]{sun2021arxiv}.\footnote{Strictly speaking, the objective used in \citep{sun2021arxiv} is $\widehat{\rdv{S}}_{n}(a,b)/\beta$, but all key results easily translate to our setup.} The corresponding $\rdv{B}_{n}$ of course cannot give us sub-Gaussian estimates of the variance under such weak assumptions, but it does scale in a desirable way \citep[\S{3.2}]{sun2021arxiv}, and when bias is mitigated by setting $\beta$ sufficiently small given the sample size $n$, the resulting $\rdv{B}_{n}$ is good enough to provide such guarantees for $\rdv{A}_{n}$, which is the ultimate goal anyways. By taking on a slightly more difficult optimization problem, it is possible to get away with not having prior knowledge or sub-Gaussian estimates of the variance. We use this basic insight as a stepping stone to our approach for learning algorithms charged with selecting a decision $h$ such that the loss $\loss(h)$ has a small mean-variance.

\subsection{A bridge between two problems}\label{sec:background_bridge}

To develop our proposal, we now return to the more general learning setup, where the test data is a random vector $\rdv{Z} \sim \ddist$, test loss is $\loss(h) \defeq \ell(h;\rdv{Z})$, and we have $n$ IID training points $\rdv{Z}_{1},\ldots,\rdv{Z}_{n}$ yielding losses $\loss_{i}(\cdot) \defeq \ell(\cdot;\rdv{Z}_{i})$, $i \in [n]$. If our goal was to simply minimize the traditional risk $\exx_{\ddist}\loss(h)$ over $h \in \HH$ under heavy-tailed losses, then in principle we could extend the approach of \S{\ref{sec:background_sun}} to robustly estimate the test risk using
\begin{align}\label{eqn:background_sun_brownlees_prep}
(\rdv{A}_{n}(h),\rdv{B}_{n}(h)) \in \argmin_{a \in \RR, b > 0} \left[ \beta b + \frac{b}{n}\sum_{i=1}^{n} \rho\left(\frac{\loss_{i}(h)-a}{b}\right) \right]
\end{align}
and design a learning algorithm using (\ref{eqn:background_sun_brownlees_prep}) as follows:
\begin{align}\label{eqn:background_sun_brownlees}
\rdv{H}_{n} \in \argmin_{h \in \HH} \rdv{A}_{n}(h).
\end{align}
Under some regularity conditions, the machinery of \citet{brownlees2015a} could then be combined with pointwise concentration inequalities in \citep{sun2021arxiv} to control the tails of $\exx_{\ddist}\loss(\rdv{H}_{n})$ under just finite loss variance. Our goal however is not to minimize the expected loss, but rather the mean-SD sum (\ref{eqn:mean_stdev_defn}). Furthermore, the bi-level program inherent in (\ref{eqn:background_sun_brownlees}) is not computationally congenial from the perspective of large-scale machine learning tasks. To ease the computational burden while at the same time building a bridge between these two problems, we consider a new objective function taking the form
\begin{align}\label{eqn:crit_empirical}
\widehat{\rdv{C}}_{n}(h;a,b) & \defeq \alpha a + \beta b + \frac{\lambda b}{n} \sum_{i=1}^{n} \rho\left(\frac{\loss_{i}(h) - a}{b}\right)
\end{align}
with parameters $\alpha \geq 0$ and $\beta \geq 0$. We call (\ref{eqn:crit_empirical}) the \term{modified Sun-Huber objective}, since $\rho$ is fixed as (\ref{eqn:sun_huber_fn}), and this form plays a special role in our analysis. Compared with that of (\ref{eqn:background_sun_brownlees}), this objective is a simple function of $h$, and gradient-based minimizers can be easily applied assuming the underlying loss $\ell(\cdot)$ is sufficiently smooth. On the other hand, it is ``biased'' in the sense that it penalizes not just the loss location (whenever $\alpha > 0$), but the loss scale as well (whenever $\beta > 0$). Intuitively, some kind of deviation-driven ``bias'' is precisely what we need from the standpoint of minimizing the mean-SD objective $\msd_{\ddist}(h;\lambda)$, but it is not immediately clear how this objective relates to $\widehat{\rdv{C}}_{n}(h;a,b)$, and it is equally unclear if we can just plug this new objective into standard machine learning workflows (e.g., using stochastic gradient-based optimizers) and achieve the desired effect without a prohibitive amount of tuning.

\subsection{Overview of contributions and limitations}\label{sec:background_ours}

With our basic idea described and some key questions raised, we summarize the central points that characterize the rest of this paper, and also highlight the limitations of this work. Broadly speaking, the new proposal here is a class of empirical ``risk'' minimizers, namely any learning algorithm which minimizes the new empirical objective (\ref{eqn:crit_empirical}). More explicitly, this refers to all procedures which returns a triplet satisfying
\begin{align}\label{eqn:our_learner}
(\rdv{H}_{n},\rdv{A}_{n},\rdv{B}_{n}) \in \argmin_{h \in \HH, a \in \RR, b > 0} \widehat{\rdv{C}}_{n}(h;a,b)
\end{align}
where $\HH$ denotes a set of feasible decisions, and we note that each element of this class is characterized by the settings of $\alpha$, $\beta$, and $\lambda$ used to define $\widehat{\rdv{C}}_{n}$. In analogy with the strategy employed in \S{\ref{sec:background_sun}}, we do not expect $\rdv{A}_{n}$ and $\rdv{B}_{n}$ to provide sub-Gaussian estimates; we simply hope that these estimates are good enough to ensure the mean-SD is smaller and/or better-behaved when compared to standard benchmarks such as mean-based empirical risk minimization (ERM) and DRO-based algorithms. Theoretically, we are interested in identifying links between the proposed objective $\widehat{\rdv{C}}_{n}$ and loss properties such as $\exx_{\ddist}\loss(h)$ and $\vaa_{\ddist}\loss(h)$, with particular emphasis on how the settings of $\alpha$, $\beta$, and $\lambda$ influence such links.

Our main theory-driven contribution is the derivation of a principled approach to determine $\widehat{\rdv{C}}_{n}$ (i.e., set $\alpha$ and $\beta$), before seeing any training data, in such a way that we can balance between ``biased but robust'' $\rho$-based deviations and ``unbiased but outlier-sensitive'' squared deviations that arise in the loss variance. Details are in \S{\ref{sec:theory_link_msd}}--\S{\ref{sec:theory_finite}}, and a concise procedure is summarized in Algorithm \ref{algo:ours}. We do not, however, consider the behavior of $\msd_{\ddist}(\rdv{H}_{n};\lambda)$ for a particular implementation of (\ref{eqn:our_learner}) (e.g., SGD) from a theoretical viewpoint; the implementation is left abstract. This is where the empirical analysis of \S{\ref{sec:empirical}} comes in. We provide evidence using simulated and real data that our procedure can be quite useful, even using a rudimentary implementation where we wrap base loss objects and naively pass them to standard stochastic gradient-based learning routines, with no manual tweaking of parameters.

\section{Theory}\label{sec:theory}

\subsection{Links to the mean-SD objective}\label{sec:theory_link_msd}

We would like to make the connection between the proposed objective (\ref{eqn:crit_empirical}) and the ultimate objective (\ref{eqn:mean_stdev_defn}) a bit more transparent. To do this, we will make use of the population version of $\widehat{\rdv{C}}_{n}$, denoted henceforth by $\crit_{\ddist}$ and defined as
\begin{align}\label{eqn:crit_population}
\crit_{\ddist}(h;a,b) \defeq \alpha{a} + \beta{b} + \lambda{b}\exx_{\ddist}\rho\left(\frac{\loss(h)-a}{b}\right).
\end{align}
Let us fix the decision $h$ and threshold $a$, paying close attention to the optimal value of the scale $b$, denoted here by $b_{\ddist}(h,a)$. More explicitly, consider any positive real number satisfying
\begin{align}\label{eqn:optimal_scale}
b_{\ddist}(h,a) \in \argmin_{b > 0} \crit_{\ddist}(h;a,b).
\end{align}
While it is not explicit in our notation, the optimal scale in (\ref{eqn:optimal_scale}) depends critically on the value of $\beta$. Intuitively, a smaller value of $\beta$ leads to a weaker penalty for taking $b$ large, thus encouraging a larger value of $b_{\ddist}(h,a)$. In fact, one can show that viewing $b_{\ddist}(h,a)$ as a function of the parameter $\beta$, in the limit we have
\begin{align}\label{eqn:beta_makes_scale_grow}
\lim\limits_{\beta \to 0_{+}} b_{\ddist}(h,a) = \infty.
\end{align}
See \S{\ref{sec:appendix_proofs}} for a proof of (\ref{eqn:beta_makes_scale_grow}). Combining this with the fact (also proved in \S{\ref{sec:appendix_proofs}}) that
\begin{align}\label{eqn:robust_deviations_vanish}
\lim\limits_{b \to \infty} b \exx_{\ddist}\rho\left(\frac{\loss(h)-a}{b}\right) = 0
\end{align}
also holds, by re-scaling to avoid trivial limits we can obtain a result which sharply bounds the proposed learning criterion at the optimal scale using the square root of \emph{quadratic} deviations, thereby establishing a clear link to the desired mean-SD objective (\ref{eqn:mean_stdev_defn}).
\begin{prop}\label{prop:crit_optimized_scale}
Let $\HH$ be such that $\exx_{\ddist}\abs{\loss(h)}^{2} < \infty$ for each $h \in \HH$. If we set $\alpha = \alpha(\beta)$ such that $\alpha(\beta)/\sqrt{\beta} \to \widetilde{\alpha} \in [0,\infty)$ as $\beta \to 0_{+}$, then in this limit, with appropriate re-scaling the scale-optimized learning criteria can be bounded above and below as
\begin{align*}
\widetilde{\alpha}a & + (1/2)\sqrt{\lambda\exx_{\ddist}(\loss(h)-a)^{2}}\\
& \leq \lim_{\beta \to 0_{+}} \min_{b > 0}\frac{\crit_{\ddist}(h;a,b)}{\sqrt{\beta}}\\
& \leq \widetilde{\alpha}a + 4\sqrt{\lambda\exx_{\ddist}(\loss(h)-a)^{2}}
\end{align*}
for any choice of threshold $a \in \RR$ and weight $\alpha \geq 0$.
\end{prop}
\noindent
In the special case where $a=\exx_{\ddist}\loss(h)$ and $\widetilde{\alpha} > 0$, we naturally recover mean-SD sums akin to those studied in an ERM framework by \citet{maurer2009a} and those bounded from above using convex surrogates by \citet{duchi2019a}.

Of course in practice, we will only ever be working with fixed values of $\beta$, and the entire point of introducing new criteria (namely $\widehat{\rdv{C}}_{n}$ and $\crit_{\ddist}$) was to give us some control over how sensitive our objective is to loss tails. The following result makes the nature of this control (through $\beta$) more transparent.
\begin{prop}\label{prop:motivator_scale}
Let $\HH$ and $\loss(h)$ be as stated in Proposition \ref{prop:crit_optimized_scale}. Letting $b_{\ddist}(h,a)$ be as specified in (\ref{eqn:optimal_scale}), we define a Bernoulli random variable
\begin{align*}
\rdv{I}(h;a) \defeq \indic\left\{\abs{\loss(h)-a} \leq b_{\ddist}(h,a)\right\}
\end{align*}
for any choice of $h \in \HH$ and $a \in \RR$. The optimal scale can then be bounded by
\begin{align*}
\frac{\lambda}{4\beta}\exx_{\ddist}\rdv{I}(h;a)(\loss(h)-a)^{2} \leq b_{\ddist}^{2}(h,a) \leq \frac{\lambda}{2\beta}\exx_{\ddist}(\loss(h)-a)^{2}
\end{align*}
for any choice of $0 < \beta < \lambda$ and $a \in \RR$.
\end{prop}
\noindent
While it is difficult to pin down exactly how $b_{\ddist}(h,a)$ changes as a function of $\beta$, Proposition \ref{prop:motivator_scale} clearly shows us the appealing property that optimal scale induced by the proposed objective function essentially falls between the (tail-sensitive) quadratic deviations and a (tail-insensitive) truncated variant, with the truncation threshold loosening as $\beta$ shrinks.

\subsection{Guiding the optimal threshold}\label{sec:theory_location}

Since the preceding Propositions \ref{prop:crit_optimized_scale}--\ref{prop:motivator_scale} both hold for any choice of threshold $a \in \RR$, they clearly hold when both $a$ and $b$ are optimal, i.e., when $a$ and $b$ are set as
\begin{align}\label{eqn:optimal_pair}
(a_{\ddist}(h),b_{\ddist}(h)) \in \argmin_{a \in \RR, b>0} \crit_{\ddist}(h;a,b).
\end{align}
In particular, using first-order conditions, the inclusion (\ref{eqn:optimal_pair}) is equivalent to the next two equalities holding:
\begin{align}
\nonumber
\exx_{\ddist}\left(\frac{\loss(h)-a_{\ddist}(h)}{\sqrt{(\loss(h)-a_{\ddist}(h))^{2}+b_{\ddist}^{2}(h)}}\right) & = \alpha/\lambda,\\
\label{eqn:pair_optimality}
 \exx_{\ddist}\left(\frac{b_{\ddist}(h)}{\sqrt{(\loss(h)-a_{\ddist}(h))^{2}+b_{\ddist}^{2}(h)}}\right) & = 1 - (\beta/\lambda).
\end{align}
Given the context of our analysis in \S{\ref{sec:theory_link_msd}}, let us consider the effect of taking $\beta$ towards zero. For any non-trivial random loss, the second equality asks that $b_{\ddist}(h)$ grow without bound as $\beta \to 0_{+}$, while $\abs{a_{\ddist}(h)}$ must be either bounded or grow slower than $b_{\ddist}(h)$. On the other hand, if $\alpha$ is too large (i.e., $\alpha > \lambda$) then the first equality will be impossible to satisfy. In addition to taking $0 < \alpha < \lambda$, note that if we multiply both sides of the first equality in (\ref{eqn:pair_optimality}) by $b_{\ddist}(h)$ and apply Proposition \ref{prop:motivator_scale}, then under this optimality condition we must have
\begin{align}\label{eqn:location_condition_upbd}
\exx_{\ddist}\left(\frac{\loss(h)-a_{\ddist}(h)}{\sqrt{(\frac{\loss(h)-a_{\ddist}(h)}{b_{\ddist}(h)})^{2}+1}}\right) \leq \alpha \sqrt{\frac{\lambda}{2\beta}\exx_{\ddist}(\loss(h)-a_{\ddist}(h))^{2}}.
\end{align}
With this inequality in place, we adopt the following strategy: encourage the optimal location to converge as $a_{\ddist}(h) \to \exx_{\ddist}\loss(h)$ when $\beta \to 0_{+}$. Since $\lambda > 0$ is assumed to be fixed in advance, the only way to ensure this using (\ref{eqn:location_condition_upbd}) is to set $\alpha = \alpha(\beta)$ such that
\begin{align}\label{eqn:alpha_condition}
\lim\limits_{\beta \to 0_{+}} \frac{\alpha(\beta)}{\sqrt{\beta}} = 0.
\end{align}
While (\ref{eqn:alpha_condition}) gives us a rather clear condition for determining $\alpha$ given $\beta$, we still do not have a principled setting for $\beta$. This point will be treated next.

\subsection{Deriving an algorithm for finite samples}\label{sec:theory_finite}

To complement the preceding analysis and discussion centered around the population objective (\ref{eqn:crit_population}), we now return to the empirical objective function $\widehat{\rdv{C}}_{n}(h;a,b)$ introduced in (\ref{eqn:crit_empirical}). We maintain the running assumption that the training data $\rdv{Z}_{1},\ldots,\rdv{Z}_{n}$ are an IID sample from $\ddist$, and thus the losses $\loss_{i}(h)$, $i=1,\ldots,n$ are independent given any fixed $h$. With $h$ and $b>0$ fixed for the moment, we will now take a closer look at the optimal (empirical) threshold that arises from this objective function, namely any random variable $\rdv{A}_{n}(h,b)$ satisfying
\begin{align}\label{eqn:optimal_empirical_location}
\rdv{A}_{n}(h,b) \in \argmin_{a \in \RR} \widehat{\rdv{C}}_{n}(h;a,b).
\end{align}
Using the property (\ref{eqn:motivation_catoni_condition}) of the smooth Huber-like function $\rho$, we can demonstrate how data-driven thresholds satisfying (\ref{eqn:optimal_empirical_location}) are concentrated at a point near the expected loss, where $\alpha$ and $b$ play a key role in how close this point is to the mean.
\begin{prop}[Concentration at a shifted location]\label{prop:location_concentration}
Taking $0 \leq \alpha < 1$, $b > 0$, and $0 < \delta < 1$, with large enough $n$ it is always possible to satisfy the condition
\begin{align*}
\frac{4\alpha}{\lambda} \leq 4\left( \frac{\vaa_{\ddist}\loss(h)}{b^{2}} + \frac{\log(2/\delta)}{n} \right) \leq 1 - \frac{4\alpha}{\lambda},
\end{align*}
and when this condition is satisfied, the data-driven threshold $\rdv{A}_{n}(h,b)$ in (\ref{eqn:optimal_empirical_location}) satisfies
\begin{align*}
& \left\lvert \rdv{A}_{n}(h,b) - \left[\exx_{\ddist}\loss(h)- \frac{2\alpha}{\lambda}b\right] \right\rvert\\
& \leq 2\left( \frac{\vaa_{\ddist}\loss(h)}{b} + \frac{b\log(2/\delta)}{n} \right)
\end{align*}
with probability no less than $1-\delta$.
\end{prop}
\noindent
This result can be seen as an extension of \citep[Prop.~3.1]{sun2021arxiv} for the function (\ref{eqn:motivation_sun_obj}) used in mean estimation to our generalized learning problem, although we use a different proof strategy which does not require strong convexity of $\widehat{\rdv{C}}_{n}$ (with respect to $a$).

With Proposition \ref{prop:location_concentration} established, conventional wisdom might incline one to pursue a $\bigO(1/\sqrt{n})$ rate in the upper bound; in this case, setting $\beta \propto 1/n$ is a natural strategy since Proposition \ref{prop:motivator_scale} tells us that for the population objective, the optimal setting of $b$ scales with $\sqrt{\lambda/\beta}$. While this is natural from the perspective of tight concentration bounds for $\rdv{A}_{n}(h,b)$, we argue that a different strategy is more appropriate when we actually consider how $(\rdv{H}_{n},\rdv{A}_{n},\rdv{B}_{n})$ will behave in the full joint optimization (\ref{eqn:our_learner}). The most obvious reason for this is that the joint objective lacks convexity and smoothness, as the following result summarizes.
\begin{prop}[Joint objective is non-convex and non-smooth]\label{prop:nonsmooth}
Even when $\HH$ is a compact convex set and the base loss function $\ell(\cdot;\rdv{Z})$ is convex, the mapping $(h,a,b) \mapsto \widehat{\rdv{C}}_{n}$ is not convex in general, and is non-smooth in the sense that its gradient is not Lipschitz continuous on $\HH \times \RR \times (0,\infty)$.
\end{prop}
\noindent
In consideration of Proposition \ref{prop:nonsmooth}, standard complexity results for typical optimizers such as stochastic gradient descent to achieve a $\varepsilon$-stationary point are on the order of $\bigO(\varepsilon^{-4})$; see \citet{davis2019b} for example.\footnote{Even if the objective were smooth, the same rates are typical; see for example \citet{ghadimi2013a}.} With this in mind, setting $\beta \propto 1/n$ to achieve $\bigO(\varepsilon^{-2})$ sample complexity for error bounds of $\rdv{A}_{n}(h,b)$ seems superfluous if in the end the dominant complexity for solving the ultimate problem (\ref{eqn:our_learner}) will be of the order $\bigO(\varepsilon^{-4})$. As such, in order to match this rate, the more natural strategy is to set $\beta \propto 1/\sqrt{n}$, or more precisely to set
\begin{align}\label{eqn:beta_condition}
\beta = \frac{\beta_{0}}{\sqrt{n}}
\end{align}
where $\beta_{0} > 0$ is a constant used to ensure $0 < \beta < \lambda$. This, coupled with $\alpha(\beta) = \beta$ to satisfy (\ref{eqn:alpha_condition}) from the previous sub-section, is our proposed setting to determine $(\alpha,\beta)$ (and thus $\widehat{\rdv{C}}_{n}$) using just knowledge of $n$, and without having observed any data points. This procedure is summarized in Algorithm \ref{algo:ours}, and will be studied empirically in \S{\ref{sec:empirical}}.

\begin{algorithm}[t]
\caption{Modified Sun-Huber}
\label{algo:ours}
\begin{algorithmic}
\STATE{{\bfseries Inputs:} data $\rdv{Z}_{1},\ldots,\rdv{Z}_{n}$ and parameter $\lambda > 0$.}
\smallskip
\STATE{{\bfseries Set:} $\displaystyle \beta = \beta_{0} / \sqrt{n}$, with $\beta_{0}$ such that $0 < \beta < \lambda$.}\hfill\COMMENT{Based on (\ref{eqn:beta_condition}).}
\smallskip
\STATE{{\bfseries Set:} $\displaystyle \alpha = \beta$.}\hfill\COMMENT{Satisfies (\ref{eqn:alpha_condition}).}
\smallskip
\STATE{{\bfseries Minimize:} $\displaystyle \widehat{\rdv{C}}_{n}(h;a,b)$ in $(h,a,b)$ using $\alpha$ and $\beta$ as above.}
\end{algorithmic}
\end{algorithm}

\subsection{Stationary points of mean-variance}

Having established links between the proposed objective and the mean-SD objective, we next consider the mean-variance objective
\begin{align}\label{eqn:mean_variance_defn}
\mv_{\ddist}(h) \defeq \exx_{\ddist}\loss(h) + \vaa_{\ddist}\loss(h).
\end{align}
This quantity can be expressed as the minimum value of a convex function, namely we have
\begin{align}\label{eqn:mv_convex_opt_value}
\mv_{\ddist}(h) = \min_{a \in \RR} \left[ a + \frac{\exx_{\ddist}(\loss(h)-a)^{2}+1}{2} \right] = a_{\mv}(h) + \frac{\exx_{\ddist}(\loss(h)-a_{\mv}(h))^{2}+1}{2}
\end{align}
where on the right-most side we have set $a_{\mv}(h) \defeq \exx_{\ddist}\loss(h)-1$. Assuming the underlying loss is differentiable, the gradient with respect to $h$ can be written as
\begin{align*}
\mv_{\ddist}^{\prime}(h) & = \exx_{\ddist}\loss^{\prime}(h) + \exx_{\ddist}\loss(h)\loss^{\prime}(h) - \exx_{\ddist}\loss(h)\exx_{\ddist}\loss^{\prime}(h)\\
& = \exx_{\ddist}\loss^{\prime}(h) + \exx_{\ddist}\left(\loss(h) - \exx_{\ddist}\loss(h)\right)\loss^{\prime}(h)\\
& = \exx_{\ddist}\left(\loss(h) - \left(\exx_{\ddist}\loss(h)-1\right)\right)\loss^{\prime}(h)
\end{align*}
which implies a stationarity condition of
\begin{align}
\mv_{\ddist}^{\prime}(h) = 0 \iff \exx_{\ddist}\left(\loss(h) - \left(\exx_{\ddist}\loss(h)-1\right)\right)\loss^{\prime}(h) = 0.
\end{align}
Similarly, the partial derivative of the learning criterion (\ref{eqn:crit_population}) taken with respect to $h$ is
\begin{align*}
\frac{\partial}{\partial{h}}\crit_{\ddist}(h;a,b) = \exx_{\ddist}\left(\frac{\loss(h)-a}{\sqrt{(\loss(h)-a)^{2}+b^{2}}}\right)\loss^{\prime}(h)
\end{align*}
and thus multiplying both sides by $b > 0$, we obtain a simple stationarity condition of
\begin{align}
\frac{\partial}{\partial{h}}\crit_{\ddist}(h;a,b) = 0 \iff \exx_{\ddist}\left(\frac{\loss(h)-a}{\sqrt{(\frac{\loss(h)-a}{b})^{2}+1}}\right)\loss^{\prime}(h) = 0.
\end{align}
With the right threshold setting, obviously the two conditions become very similar as $b$ grows large. The following result makes this precise.
\begin{prop}
Let loss function $\ell$ and data distribution $\ddist$ be such that the random vector $\loss(h)\loss^{\prime}(h)$ is integrable and has a norm with finite mean, i.e., $\exx_{\ddist}\Abs{\loss(h)\loss^{\prime}(h)} < \infty$ for some choice of $h \in \HH$. Then, for any $a \in \RR$, defining
\begin{align}
f(h;a) \defeq \lim\limits_{b \to \infty} b \frac{\partial}{\partial{h}}\crit_{\ddist}(h;a,b)
\end{align}
the stationary points of the mean-variance objective are related to those of the proposed objective (\ref{eqn:crit_population}) through the following equivalence:
\begin{align*}
f(h;a_{\mv}(h)) = 0 \iff \frac{\partial}{\partial{h}}\mv_{\ddist}(h) = 0
\end{align*}
where $\mv_{\ddist}(h)$ is as defined in (\ref{eqn:mean_variance_defn}).
\end{prop}

\subsection{Comparison with dual form of DRO risk}

Some readers may notice that the proposed (population) objective (\ref{eqn:crit_population}) looks quite similar to the dual form of DRO risks:
\begin{align}
\text{DRO}_{\ddist}(h;\beta) \defeq \inf_{a \in \RR, b > 0} \left[ a + \beta b + b \exx_{\ddist}\phi^{\ast}\left(\frac{\loss(h)-a}{b}\right)  \right]
\end{align}
where $\phi^{\ast}$ is the Legendre-Fenchel convex conjugate $\phi^{\ast}(x) \defeq \sup_{u \in \RR} [xu - \phi(u)]$ induced by a function $\phi:\RR \to \overbar{\RR}$, assumed to be convex and lower semi-continuous, with $\phi(1)=0$ and $\phi(x) = \infty$ whenever $x < 0$ (cf.~\citep[\S{3.2}]{shapiro2017a}). Given this similarity, one might ask whether or not some form of DRO risk can be reverse engineered from our proposed objective. Taking up this point briefly, we first note that the conjugate of $\rho$ given by (\ref{eqn:sun_huber_fn}) is
\begin{align*}
\rho^{\ast}(x) \defeq \sup_{u \in \RR} \left[xu - \rho(u)\right] = \sup_{u \in \RR} \left[xu - \sqrt{u^{2}+1} + 1\right].
\end{align*}
From the non-negative nature of $\rho$, clearly $\rho^{\ast}(0) = -\rho(0) = 0$. For $x \neq 0$, note that taking the derivative of concave function $u \mapsto xu - \rho(u)$ and setting it to zero, we obtain the first-order optimality conditions
\begin{align*}
\frac{u}{\sqrt{u^{2}+1}} = x \iff \frac{\sign(x)}{\sqrt{1+1/u^{2}}} = x \iff \frac{1}{x^{2}} = 1+1/u^{2} \iff u = \frac{\sign(x)}{\sqrt{1/x^{2} - 1}}.
\end{align*}
Plugging this solution in whenever $\abs{x} < 1$ and doing a bit of algebra readily yields the simple closed-form expression
\begin{align}\label{eqn:rho_legendre}
\rho^{\ast}(x) = 
\begin{cases}
\frac{x^{2}}{\sqrt{1-x^{2}}} + 1 - \frac{1}{\sqrt{1-x^{2}}}, & \text{ if } 0 \leq \abs{x} < 1\\
\infty, & \text{ else}.
\end{cases}
\end{align}
As can be readily observed from both (\ref{eqn:rho_legendre}) and Figure \ref{fig:rho_legendre}, this function does not satisfy any of the requirements placed on $\phi$ except convexity, and thus despite the similar form, the non-monotonic nature of $\rho$ is in sharp contrast with monotonicity of typical cases of $\phi^{\ast}$ that arise in the DRO literature (e.g.~\citep[\S{3}]{bental2013a}), and does not readily imply a ``primal'' DRO objective that can be recovered using $\rho^{\ast}$.

\begin{figure}[t]
\centering
\includegraphics[width=0.4\textwidth]{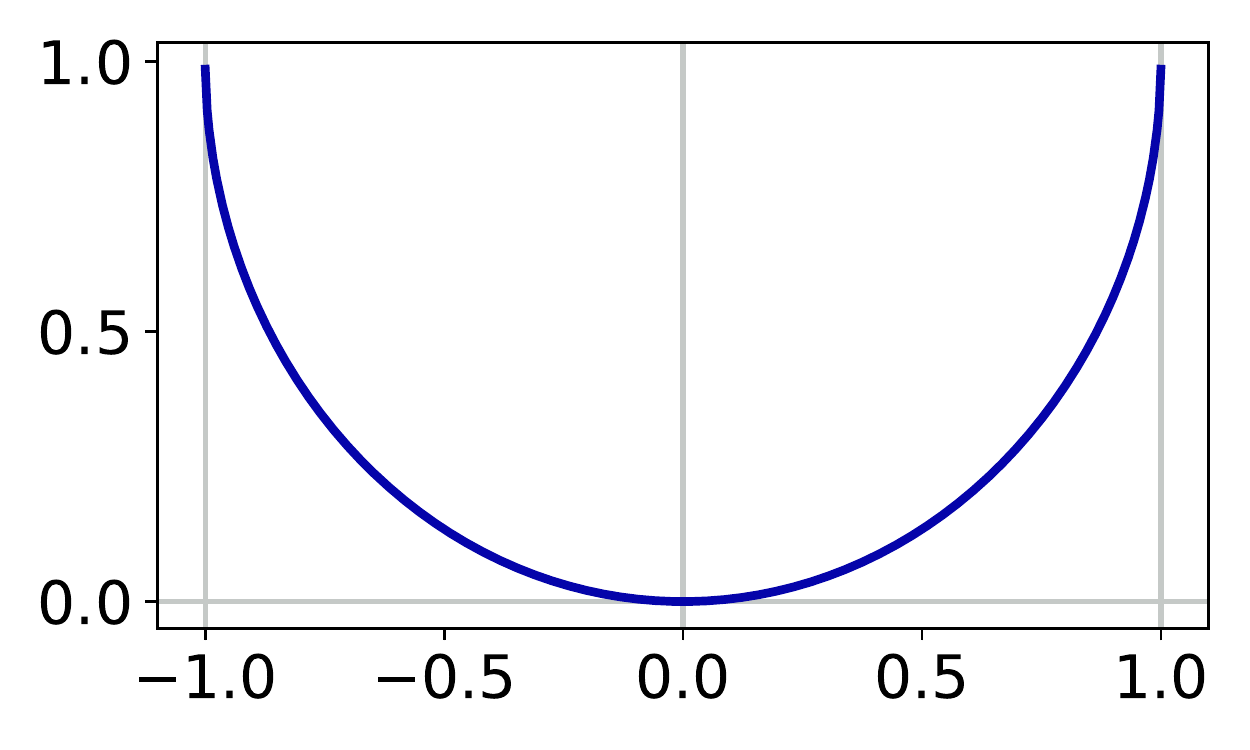}
\caption{Graph of the Legendre transform $\rho^{\ast}$ as given in (\ref{eqn:rho_legendre}) over $(-1,1)$.}
\label{fig:rho_legendre}
\end{figure}

\section{Empirical analysis}\label{sec:empirical}

Our investigation in the previous section led us to Algorithm \ref{algo:ours}, giving us a principled and precise strategy to construct the objective function $\widehat{\rdv{C}}_{n}$, but leaving the actual minimization procedure abstract. Here we make this concrete by implementing a simple gradient-based minimizer of this objective, and comparing this procedure with natural benchmarks from the literature.

\subsection{Methods to be compared}\label{sec:empirical_methods}

In the experiments to follow in \S{\ref{sec:empirical_sims}}--\S{\ref{sec:empirical_real}}, we compare our proposed procedure (denoted in figures as ``Modified Sun-Huber'') with three alternatives: traditional mean-based empirical risk minimization (denoted ``Vanilla ERM''), conditional value-at-risk (CVaR) \citep{curi2020a}, and the well-studied $\chi^{2}$-DRO risk \citep{duchi2019a,hashimoto2018a}. For reference, here we provide the population versions of the CVaR and $\chi^{2}$-DRO criteria used in the empirical tests to follow. First, it is well known (see \citet{rockafellar2000a}) that CVaR at quantile level $\xi$ can be represented as
\begin{align}\label{eqn:cvar_convenient}
\text{CVaR}_{\ddist}(h;\xi) = \inf_{a \in \RR}\left[ a +  \frac{1}{1-\xi}\exx_{\ddist}\left(\loss(h)-a\right)_{+} \right]
\end{align}
where $(x)_{+} \defeq \max\{0,x\}$. Similarly, DRO risk based on the Cressie-Read family of divergence functions, here denoted by $\text{DRO}_{\ddist}(h;\eta)$, is formulated for any $c > 1$ and $\eta > 0$ using
\begin{align}\label{eqn:dro_convenient}
\inf_{a \in \RR}\left[ a + \left(1+c(c-1)\eta\right)^{1/c}\left(\exx_{\ddist}\left(\loss(h)-a\right)_{+}^{c_{\ast}}\right)^{1/c_{\ast}} \right]
\end{align}
where $c_{\ast} \defeq c/(c-1)$, and $\chi^{2}$-DRO is the special case where $c=2$ \citep{duchi2019a,hashimoto2018a,zhai2021a}. The different ``robustness levels'' to be mentioned in \S{\ref{sec:empirical_sims}} correspond to different values of the re-parameterized quantity $\widetilde{\eta} \in (0,1)$, related to $\eta$ by the equality $\eta = (1/(1-\widetilde{\eta})-1)/2$. Just as our $\widehat{\rdv{C}}_{n}(h;a,b)$ is solved jointly in $(h,a,b)$, our empirical tests minimize the empirical versions of (\ref{eqn:cvar_convenient}) and (\ref{eqn:dro_convenient}) jointly in $(h,a)$.

\begin{figure}[t]
\centering
\includegraphics[width=0.25\textwidth]{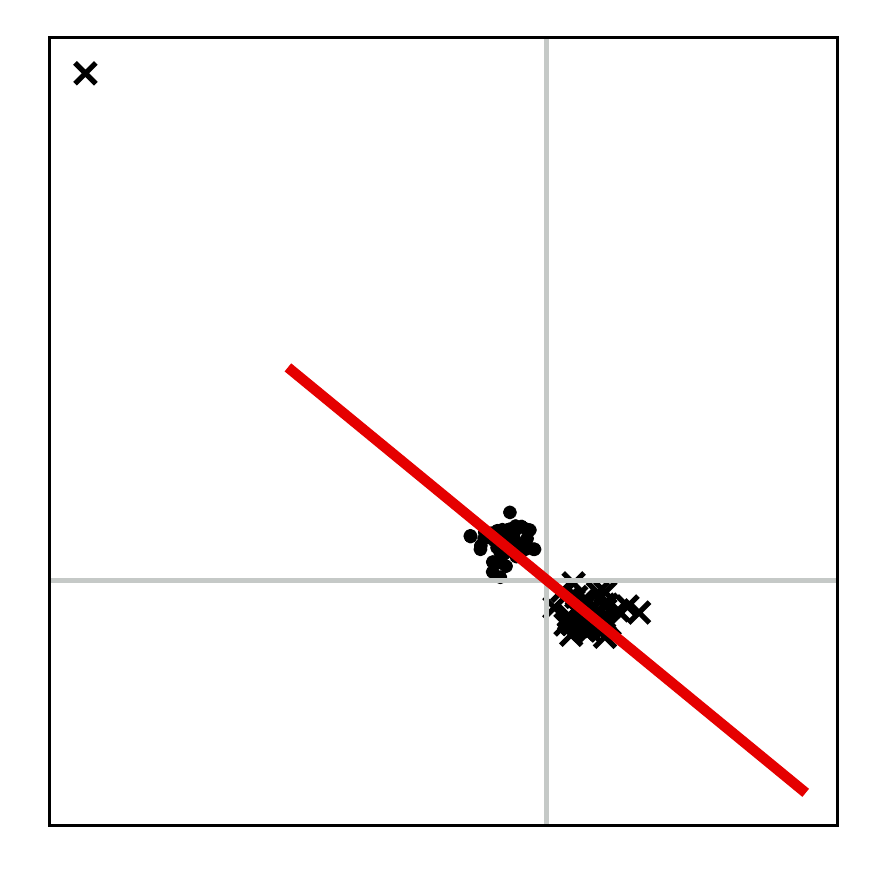}
\caption{2D classification example from \S{\ref{sec:empirical_sims}}. The red line represents the initial value used by each method.}
\label{fig:class2d_data}
\end{figure}

\begin{figure}[t]
\centering
\includegraphics[width=0.5\textwidth]{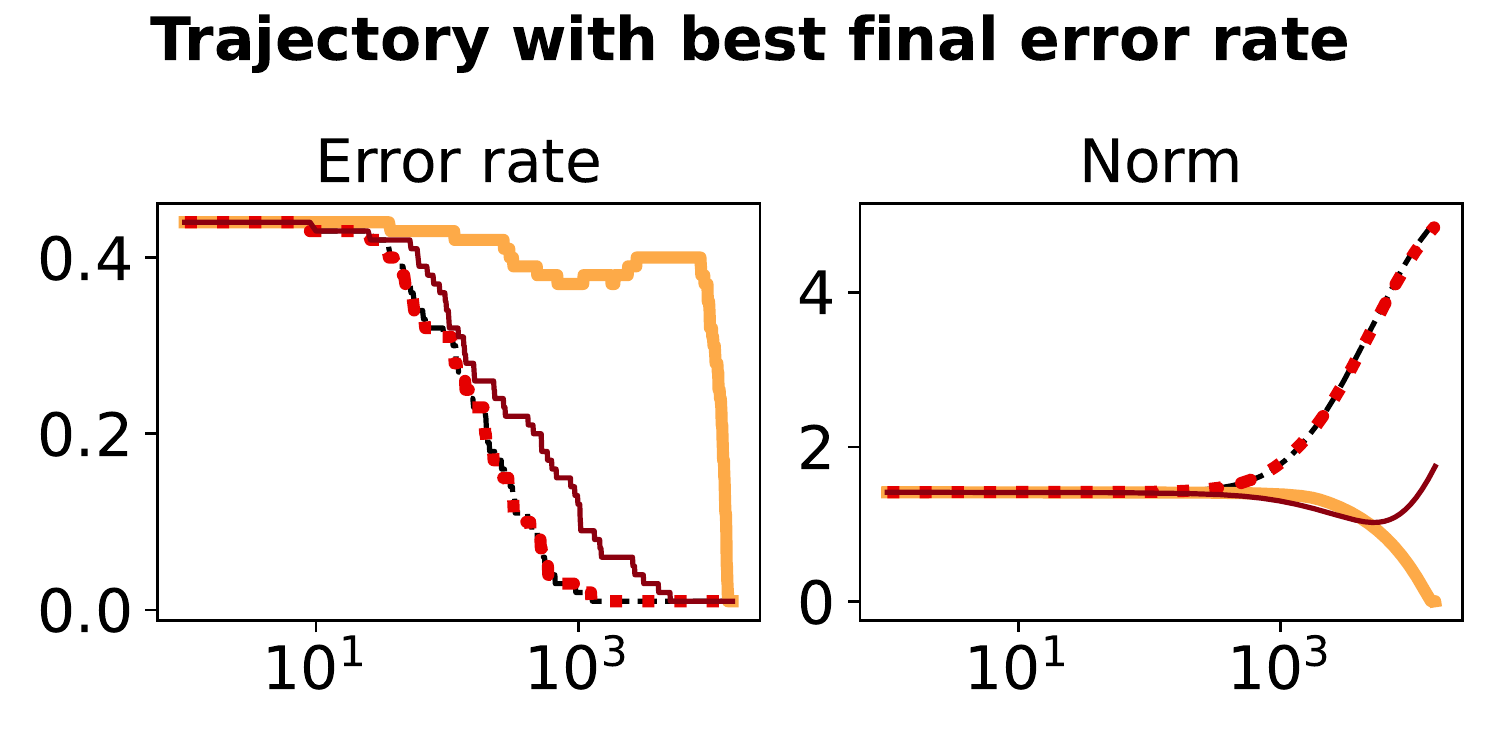}\\
\includegraphics[width=0.5\textwidth]{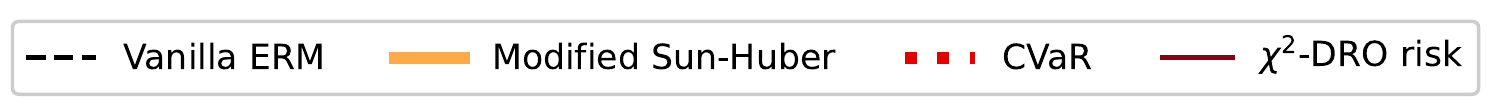}
\caption{From each method class, we show the classification error rate and Euclidean norm trajectories corresponding to the setting that achieved the best error rate after the final iteration.}
\label{fig:class2d_zeroone}
\end{figure}

\subsection{Simulated noisy classification on the plane}\label{sec:empirical_sims}

As a simplified and controlled setting to start with, we generate random data points on the plane which are mostly linearly separable, save for a single distant outlier (Figure \ref{fig:class2d_data}). Before we consider off-sample generalization, here we focus simply on the training loss distribution properties as a function of algorithm iterations.

\paragraph{Experiment setup}

We generate $n=100$ training data points using two Gaussian distributions on the plane to represent two classes, with each class having the same number of points. We choose a single point uniformly at random, and perturb it by multiplying the scalar -10. As mentioned in \S{\ref{sec:empirical_methods}}, we compare our proposed procedure (Modified Sun-Huber) with Vanilla ERM, CVaR, and $\chi^{2}$-DRO. In light of Algorithm \ref{algo:ours}, we set $\lambda = \log(n) / \sqrt{n} > \beta = \beta_{0}/\sqrt{n}$, and try a variety of $\beta_{0}$ values just for reference. For all the aforementioned methods, we set the base loss $\ell(\cdot)$ to the usual binary logistic loss (linear model), and run (batch) gradient descent on the empirical risk objectives implied by each of these methods, with a fixed step size of 0.01 over 15,000 iterations. Alternative settings of step size and iteration number were not tested. All methods are initialized at the same point, shown in Figure \ref{fig:class2d_data}.

\paragraph{Results and discussion}

In Figure \ref{fig:class2d_meansd}, we show the empirical mean-SD trajectories for the base loss, over algorithm iterations ($\log_{10}$ scale), for each method of interest. Using our notation, this is the sample version of $\msd_{\ddist}(h;1)$ in (\ref{eqn:mean_stdev_defn}); note that this differs from the $n$-dependent $\lambda$ setting that is explicit in our proposed method, and implicit in CVaR and $\chi^{2}$-DRO. All methods besides vanilla ERM have multiple settings that were tested, and the results for each are distinguished using curves of different color. Our method tests different values of $\beta_{0}$, CVaR tests different quantile levels, and DRO tests different robustness levels. Since Vanilla ERM is designed to optimize the average loss, it is perhaps not surprising that it fails in terms of the mean-SD objective. On the other hand, the proposed method (for any choice of $\beta_{0}$) is as good or better than all the competing methods. As a basic sanity check, in Figure \ref{fig:class2d_zeroone} we also consider the error rate (average zero-one loss) and model norm trajectories over iterations for each method. For each method, we plot just one trajectory, namely the one achieving the best final error rate. While our method is not designed to minimize the average loss and typical surrogate theory does not apply, the error rate is surprisingly good, albeit with slower convergence than the other methods. The error rate for CVaR matches that of Vanilla ERM; this is in fact the CVaR setting with the worst final mean-SD value. On the other hand, the proposed method performs well from both perspectives at once.

\subsection{Classification on real datasets}\label{sec:empirical_real}

We proceed to experiments using real-world datasets, some of which are orders of magnitude larger than the simple setup given in \S{\ref{sec:empirical_sims}}, and which include multi-class classification tasks.

\paragraph{Experiment setup}
We make use of four well-known datasets, all available from online repositories: \texttt{adult}, \texttt{australian}, \texttt{cifar10}, and \texttt{fashion\_mnist}. For multi-class datasets, we extend the binary logistic loss to the usual multi-class logistic regression loss under a linear model, with one linear model for each class. Features for all datasets are normalized to $[0,1]$, with one-hot representations of categorical features. The learning algorithms being compared here are the same as described in \S{\ref{sec:empirical_sims}}, except that now we implement each method using mini-batch stochastic gradient descent (batch size 32), and do 30 epochs (i.e., 30 passes over the training data). In addition, our proposed ``Modified Sun-Huber'' method performs almost identically for the range of $\beta_{0}$ values tested in \S{\ref{sec:empirical_sims}}, and thus we have simply fixed $\beta_{0} = 0.9$, so there is only one trajectory curve this time. On the other hand, we now try a range of step sizes for each method, choosing the best step size in terms of average (base) loss value on validation data for each method. We run five independent trials, and for each trial we randomly re-shuffle the dataset, taking 80\% for training, 10\% for validation (used to select step sizes), and 10\% for final testing.

\paragraph{Results and discussion}

Our main results are shown in Figure \ref{fig:real_traj_mstd} (next page), where once again we plot the trajectory of the mean-SD objective, but this time computed on test data, and given as a function of \emph{epoch number}, rather than individual iterations. Curves drawn represent averages taken over trials, and the lightly shaded region above/below each curve shows standard deviation over trials. Perhaps surprisingly, the very simple implementation of our proposed Algorithm \ref{algo:ours} (fixed step size, no regularization) works remarkably well on a number of datasets. From the perspective of mean-SD minimization, for three our of four datasets, the proposed method is far better than Vanilla ERM, and as good or better than even the best settings of CVaR and DRO viewed after the fact. Regarding the sub-standard performance observed on \texttt{fashion\_mnist}, detailed analysis shows that more fine-tuned settings of $\alpha$ and $\beta$ can readily bring the method up to par; the non-convex and non-smooth nature of $\widehat{\rdv{C}}_{n}$ naturally means that some tasks will require more careful settings than are captured by our Algorithm \ref{algo:ours}, and indeed will take explicit account of the optimizer to be used. We leave both the theoretical grounding and empirical testing of such optimizer-aligned mean-SD minimizers for future work.

\begin{figure}[t]
\centering
\includegraphics[width=\textwidth]{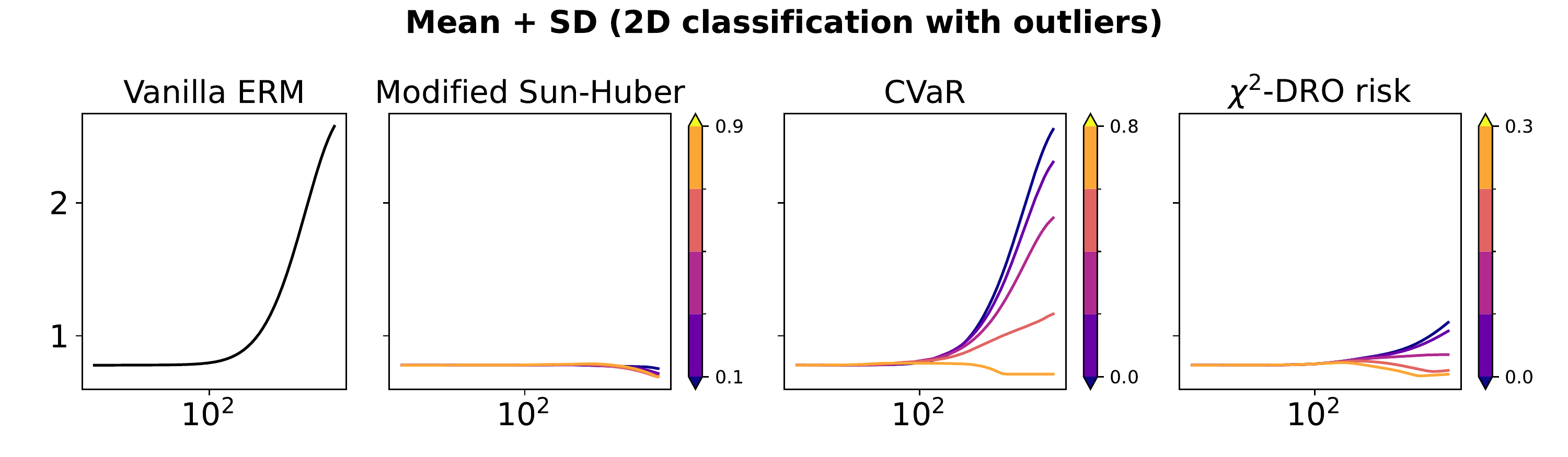}
\caption{Trajectory of the (empirical) mean-SD objective (\ref{eqn:mean_stdev_defn}) over iterations. Colors correspond to different choices from each class: $\beta_{0}$ for Modified Sun-Huber, quantile level for CVaR, and constraint level for DRO.}
\label{fig:class2d_meansd}
\end{figure}

\begin{figure*}[h!]
\centering
\includegraphics[width=\textwidth]{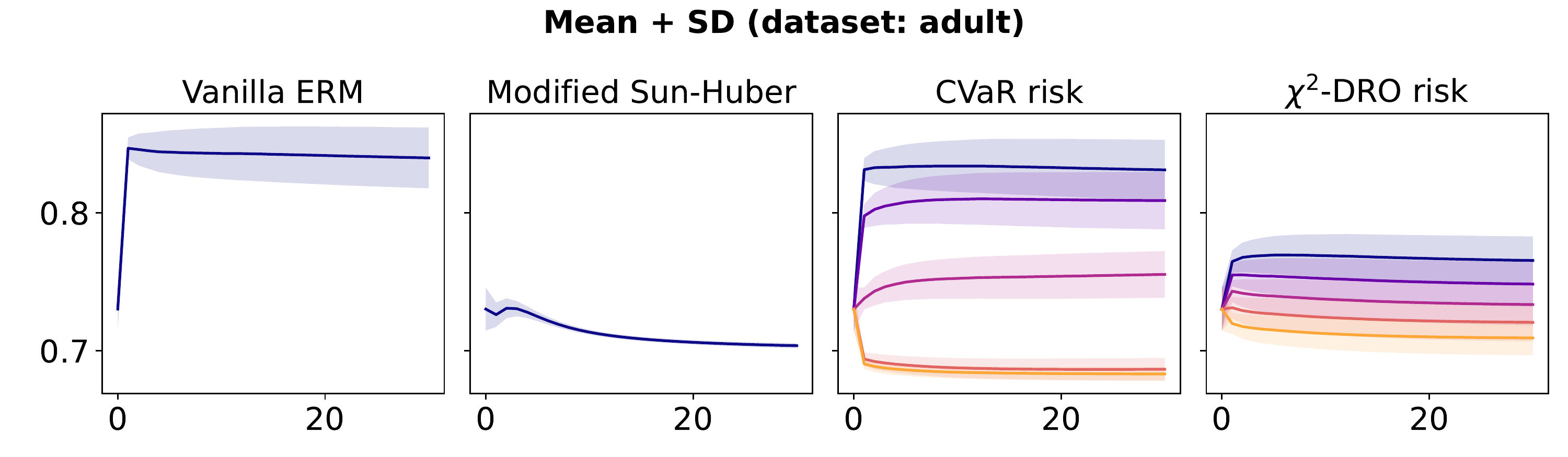}\\
\bigskip
\includegraphics[width=\textwidth]{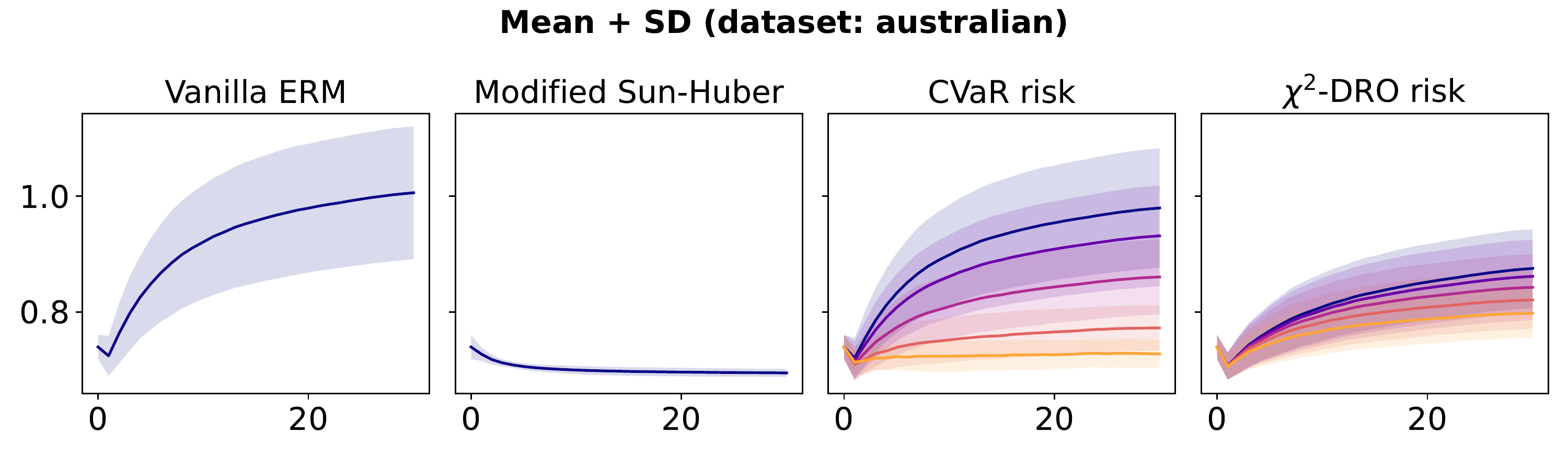}\\
\bigskip
\includegraphics[width=\textwidth]{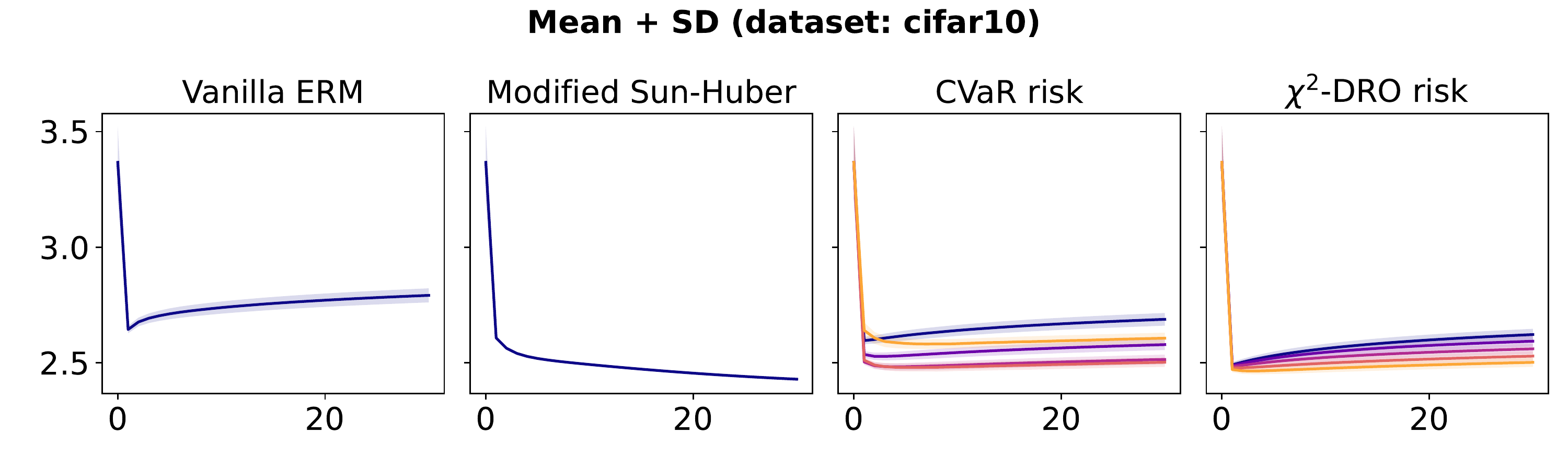}\\
\bigskip
\includegraphics[width=\textwidth]{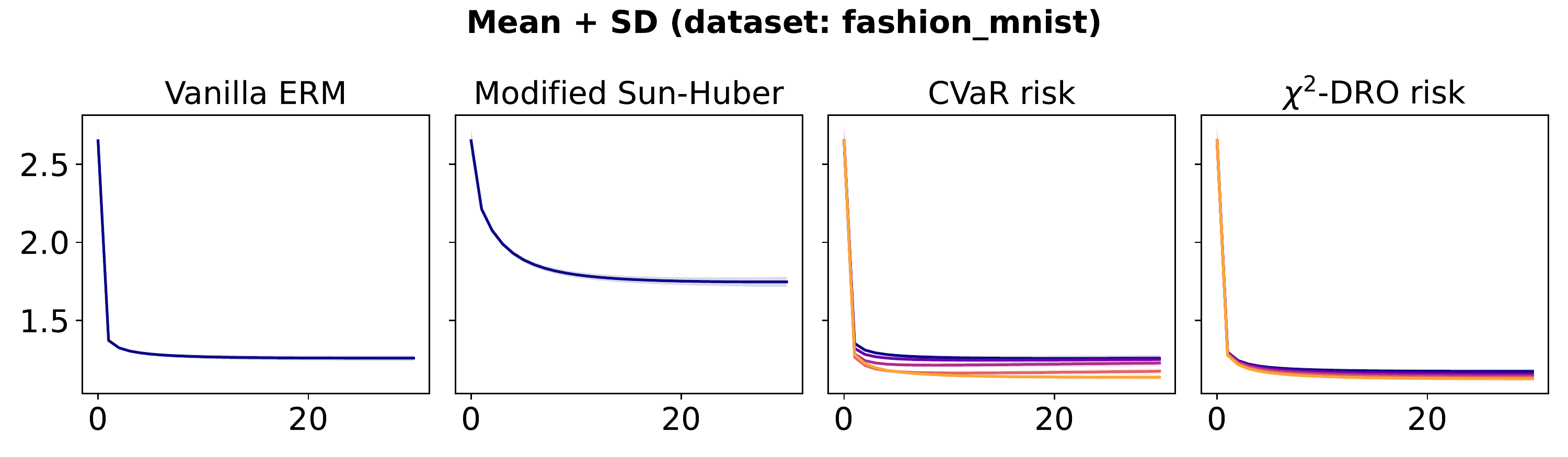}
\caption{Mean-SD trajectories on real-world datasets as described in \S{\ref{sec:empirical_real}}, given as a function of epochs and averaged over multiple independent trials. Coloring for CVaR and DRO is analogous to that of Figure \ref{fig:class2d_meansd}.}
\label{fig:real_traj_mstd}
\end{figure*}

\clearpage

\section*{Acknowledgements}
This work was supported by JSPS KAKENHI Grant Number 22H03646. Thanks also go to anonymous reviewers whose careful reading and insightful comments contributed to the final version of this paper.

\bibliography{../refs/refs}

\appendix

\section{Technical appendix}

\subsection{Basic facts}

Assuming $\rho$ is defined as in (\ref{eqn:sun_huber_fn}), let us consider the function
\begin{align}
f(x,a,b) & \defeq \alpha a + \beta b + b\rho\left(\frac{x-a}{b}\right)\\
& = \alpha a + \beta b + \sqrt{(x-a)^{2}+b^{2}} - b\\
& = \alpha a + \sqrt{(x-a)^{2}+b^{2}} - (1-\beta)b.
\end{align}
The partial derivatives are as follows.
\begin{align}
\label{eqn:partial_1_x}
\partial_{x}f(x,a,b) & = \frac{x-a}{\sqrt{(x-a)^{2}+b^{2}}}\\
\label{eqn:partial_1_a}
\partial_{a}f(x,a,b) & = \alpha - \frac{x-a}{\sqrt{(x-a)^{2}+b^{2}}}\\
\label{eqn:partial_1_b}
\partial_{b}f(x,a,b) & = \frac{b}{\sqrt{(x-a)^{2}+b^{2}}} - (1-\beta)
\end{align}
The corresponding second derivatives are as follows.
\begin{align}
\label{eqn:partial_2_xx}
\partial_{x}^{2}f(x,a,b) & = \frac{1}{\sqrt{(x-a)^{2}+b^{2}}} - \frac{(x-a)^{2}}{((x-a)^{2}+b^{2})^{3/2}} = \frac{b^{2}}{((x-a)^{2}+b^{2})^{3/2}}\\
\label{eqn:partial_2_aa}
\partial_{a}^{2}f(x,a,b) & = \frac{1}{\sqrt{(x-a)^{2}+b^{2}}} - \frac{(x-a)^{2}}{((x-a)^{2}+b^{2})^{3/2}} = \frac{b^{2}}{((x-a)^{2}+b^{2})^{3/2}}\\
\label{eqn:partial_2_bb}
\partial_{b}^{2}f(x,a,b) & = \frac{1}{\sqrt{(x-a)^{2}+b^{2}}} - \frac{b^{2}}{((x-a)^{2}+b^{2})^{3/2}} = \frac{(x-a)^{2}}{((x-a)^{2}+b^{2})^{3/2}}
\end{align}
The remaining elements of the Hessian of $f(x,a,b)$ follow easily, given as follows.
\begin{align}
\label{eqn:partial_2_ax}
\partial_{a}\partial_{x}f(x,a,b) & = \frac{-1}{\sqrt{(x-a)^{2}+b^{2}}} + \frac{(x-a)^{2}}{((x-a)^{2}+b^{2})^{3/2}} = \frac{-b^{2}}{((x-a)^{2}+b^{2})^{3/2}}\\
\label{eqn:partial_2_bx}
\partial_{b}\partial_{x}f(x,a,b) & = \frac{-b(x-a)}{((x-a)^{2}+b^{2})^{3/2}}\\
\partial_{b}\partial_{a}f(x,a,b) & = \frac{b(x-a)}{((x-a)^{2}+b^{2})^{3/2}}
\end{align}

\begin{lem}[Useful inequalities]\label{lem:helper_ineq}
\begin{align}\label{eqn:helper_ineq_rec1x}
\frac{1}{1+x} \leq 1 - \frac{x}{2}, \quad 0 \leq x \leq 1.
\end{align}
\begin{align}\label{eqn:helper_ineq_sun}
(1+x)^{c} \geq 1+cx, \quad x \geq -1, c \in \RR \setminus (0,1).
\end{align}
\end{lem}

\subsection{Convexity and smoothness}\label{sec:convexity_smoothness}

\begin{lem}\label{lem:helper_convex_recsq1}
The map $x \mapsto 1 / \sqrt{1+x}$ is convex on $[0,\infty)$.
\end{lem}

\begin{lem}[Properties of partial objective]\label{lem:partial_objective}
With $\rho$ as in (\ref{eqn:sun_huber_fn}) and $\beta \geq 0$, the function
\begin{align*}
(x,b) \mapsto \beta b + b\rho\left(\frac{x}{b}\right)
\end{align*}
is convex and $(1+\max\{1-\beta,\beta\})$-Lipschitz (in $\Abs{\cdot}_{1}$) on $\RR \times (0,\infty)$, but its gradient is not (globally) Lipschitz, and thus the function is not smooth.\footnote{We prove that the Hessian's norm is unbounded, which implies (via \citet[Thm.~2.1.6]{nesterov2004ConvOpt}) that the convex function of interest cannot be smooth.}
\end{lem}
\begin{proof}[Proof of Lemma \ref{lem:partial_objective}]
For notational convenience, setting $0 < \beta < 1$, let us denote
\begin{align*}
g(x,b) \defeq \beta b + b \rho(x/b), \qquad x \in \RR, b > 0
\end{align*}
with $\rho$ as in (\ref{eqn:sun_huber_fn}). From the partial derivatives (\ref{eqn:partial_1_x}) and (\ref{eqn:partial_1_b}), it is clear that we have
\begin{align*}
-1 \leq \partial_{x}g(x,b) \leq 1, \qquad -(1-\beta) \leq \partial_{b}g(x,b) \leq \beta
\end{align*}
when evaluated at any choice of $x \in \RR$ and $b > 0$. It follows that the gradient norm can be bounded as
\begin{align*}
\Abs{\nabla g(x,b)}_{1} \leq 1 + \max\{(1-\beta), \beta\}
\end{align*}
and thus $g(\cdot)$ is Lipschitz continuous in $\Abs{\cdot}_{1}$ (and also $\Abs{\cdot}_{2}$).\footnote{That bounded gradients imply Lipschitz continuity is a general fact on linear spaces \citep[\S{7.3}, Prop.~2]{luenberger1969Book}.}

Next, let us denote the Hessian of $g(\cdot)$ evaluated at $(x,b)$ by $H$. Basic calculus gives us the simple form
\begin{align*}
H \defeq \frac{1}{(x^{2}+b^{2})^{3/2}}
\begin{bmatrix}
b^{2} & -xb\\
-xb & x^{2}
\end{bmatrix}
\end{align*}
and for any pair of real values $\bm{u}=(u_{1},u_{2})$, we have
\begin{align}\label{eqn:basic_sun_1}
\langle H\bm{u}, \bm{u}\rangle = \frac{1}{(x^{2}+b^{2})^{3/2}}(u_{1}b - u_{2}x)^{2} \geq 0.
\end{align}
Since this holds for any choice of $x \in \RR$ and $b > 0$, the Hessian is thus positive semi-definite, implying that $g(\cdot)$ is (jointly) convex \citep[Thm.~2.1.4]{nesterov2004ConvOpt}.

On the other hand, the function $g(\cdot)$ is not smooth. To see this, first note that having chosen any $\bm{u}$ such that $\Abs{\bm{u}} \leq 1$, we have that the (operator) norm is bounded below as
\begin{align*}
\Abs{H} = \sup_{\Abs{\bm{u}^{\prime}} \leq 1} \left[ \sup_{\Abs{\bm{u}^{\prime\prime}} \leq 1} \langle H\bm{u}^{\prime}, \bm{u}^{\prime\prime} \rangle \right] \geq \langle H\bm{u}, \bm{u}\rangle.
\end{align*}
Then, as a concrete example, consider setting $x = b$, with $\bm{u}=(u_{1},u_{2})$ such that $u_{1} \neq u_{2}$. Recalling the lower bound (\ref{eqn:basic_sun_1}), we have
\begin{align*}
\|H\| \geq \frac{b^{2}}{(2b^{2})^{3/2}}(u_{1} - u_{2})^{2} = \frac{(u_{1} - u_{2})^{2}}{(\sqrt{2})^{3}b} \to \infty
\end{align*}
in the limit as $b \to 0_{+}$. As such, the gradient of $g(\cdot)$ cannot be Lipschitz continuous on $\RR \times (0,\infty)$, and thus $g(\cdot)$ is not smooth \citep[Thm.~2.1.6]{nesterov2004ConvOpt}.
\end{proof}

\section{Additional proofs}\label{sec:appendix_proofs}

\begin{proof}[Proof of Proposition \ref{prop:crit_optimized_scale}]
To begin, note that the function
\begin{align}\label{eqn:crit_optimized_scale_1}
b \mapsto b \exx_{\ddist}\rho\left(\frac{\loss(h)-a}{b}\right) = \exx_{\ddist}\left[\sqrt{(\loss(h)-a)^{2}+b^{2}}-b\right]
\end{align}
is monotonic (non-increasing) on $(0,\infty)$ (follows clearly from (\ref{eqn:partial_1_b})). We will use this property moving forward. Recalling the upper and lower bounds of Proposition \ref{prop:motivator_scale}, we re-write them as
\begin{align}\label{eqn:crit_optimized_scale_2}
\frac{c_{\text{lo}}(\beta)}{\beta}  \leq b_{\ddist}^{2}(h,a) \leq \frac{c_{\text{hi}}}{\beta}
\end{align}
using the shorthand notation
\begin{align*}
c_{\text{lo}}(\beta) & \defeq \frac{\lambda}{4}\exx_{\ddist}\rdv{I}(h;a)(\loss(h)-a)^{2}\\
c_{\text{hi}} & \defeq \frac{\lambda}{2}\exx_{\ddist}(\loss(h)-a)^{2}
\end{align*}
and noting that while $c_{\text{hi}}$ is free of $\beta$, $c_{\text{lo}}(\beta)$ depends on $\beta$ through the definition of $\rdv{I}(h;a)$. Fixing $0 < \beta < \lambda$ for now and recalling the form of $\crit_{\ddist}$ in (\ref{eqn:crit_population}), the preceding bounds (\ref{eqn:crit_optimized_scale_2}) and monotonicity of (\ref{eqn:crit_optimized_scale_1}) can be used to obtain a lower bound of the form
\begin{align}\label{eqn:crit_optimized_scale_3}
\min_{b > 0}\crit_{\ddist}(h;a,b) \geq \alpha{a} + \sqrt{\beta c_{\text{lo}}(\beta)} + \lambda\sqrt{c_{\text{hi}}}\exx_{\ddist}\left[\sqrt{\frac{(\loss(h)-a)^{2}}{c_{\text{hi}}}+\frac{1}{\beta}} - \sqrt{\frac{1}{\beta}}\right].
\end{align}
Using the fact (\ref{eqn:beta_makes_scale_grow}) and applying dominated convergence \citep[Thm.~1.6.9]{ash2000a}, in the limit we have
\begin{align*}
\lim_{\beta \to 0_{+}} c_{\text{lo}}(\beta) = \frac{\lambda}{4}\exx_{\ddist}(\loss(h)-a)^{2}.
\end{align*}
Dividing both sides of (\ref{eqn:crit_optimized_scale_3}) by $\sqrt{\beta}$, setting $\alpha = \alpha(\beta)$ as in the proposition statement, and taking the limit as $\beta \to 0_{+}$, we obtain
\begin{align*}
\lim_{\beta \to 0_{+}} \min_{b > 0}\frac{\crit_{\ddist}(h;a,b)}{\sqrt{\beta}} & \geq \widetilde{\alpha}a + \sqrt{\frac{\lambda}{4}\exx_{\ddist}(\loss(h)-a)^{2}} + \frac{\lambda\exx_{\ddist}(\loss(h)-a)^{2}}{2\sqrt{c_{\text{hi}}}}\\
& = \widetilde{\alpha}a + \sqrt{\frac{\lambda}{4}\exx_{\ddist}(\loss(h)-a)^{2}} + \sqrt{\frac{\lambda}{2}\exx_{\ddist}(\loss(h)-a)^{2}}\\
& = \widetilde{\alpha}a + \left(\frac{1}{2}+\frac{1}{\sqrt{2}}\right)\sqrt{\lambda\exx_{\ddist}(\loss(h)-a)^{2}}.
\end{align*}
The first inequality uses the fact that for any $c>0$, we have $\sqrt{cx + x^{2}}-x \to c/2$ as $x \to \infty$, and also uses dominated convergence. The remaining equalities just follow from plugging in the definition of $c_{\text{hi}}$ and cleaning up terms. This proves the desired lower bound.

As for the upper bound of interest, a perfectly analogous argument can be applied. Using Proposition \ref{prop:motivator_scale} again and taking $\beta$ small enough that
\begin{align}\label{eqn:crit_optimized_scale_4}
c_{\text{lo}}(\beta) \geq c_{\text{hi}} / 4
\end{align}
holds (always possible), we can obtain upper bounds of the form
\begin{align}
\nonumber
\min_{b > 0}\crit_{\ddist}(h;a,b) & \leq \alpha{a} + \sqrt{\beta c_{\text{hi}}} + \lambda\sqrt{c_{\text{lo}}(\beta)}\exx_{\ddist}\left[\sqrt{\frac{(\loss(h)-a)^{2}}{c_{\text{lo}}(\beta)}+\frac{1}{\beta}} - \sqrt{\frac{1}{\beta}}\right]\\
\label{eqn:crit_optimized_scale_5}
& \leq \alpha{a} + \sqrt{\beta c_{\text{hi}}} + \lambda\sqrt{c_{\text{hi}}}\exx_{\ddist}\left[\sqrt{\frac{4(\loss(h)-a)^{2}}{c_{\text{hi}}}+\frac{1}{\beta}} - \sqrt{\frac{1}{\beta}}\right]
\end{align}
noting that the latter inequality (\ref{eqn:crit_optimized_scale_5}) follows from using (\ref{eqn:crit_optimized_scale_4}) as well as $c_{\text{lo}}(\beta) \leq c_{\text{hi}}$. As with the lower bound argument in the preceding paragraph, we set $\alpha = \alpha(\beta)$, divide both sides by $\sqrt{\beta}$, and take the limit as $\beta \to 0_{+}$. This results in
\begin{align*}
\lim_{\beta \to 0_{+}} \min_{b > 0}\frac{\crit_{\ddist}(h;a,b)}{\sqrt{\beta}} & \leq \widetilde{\alpha}a + \sqrt{\frac{\lambda}{2}\exx_{\ddist}(\loss(h)-a)^{2}} + \frac{2\lambda\exx_{\ddist}(\loss(h)-a)^{2}}{\sqrt{c_{\text{hi}}}}\\
& = \widetilde{\alpha}a + \sqrt{\frac{\lambda}{2}\exx_{\ddist}(\loss(h)-a)^{2}} + 2\sqrt{2\lambda\exx_{\ddist}(\loss(h)-a)^{2}}\\
& = \widetilde{\alpha}a + \left(2\sqrt{2}+\frac{1}{\sqrt{2}}\right)\sqrt{\lambda\exx_{\ddist}(\loss(h)-a)^{2}}
\end{align*}
which gives us the desired upper bound. The bounds given in the proposition statement are slightly looser, but more readable.
\end{proof}

\bigskip

\begin{proof}[Proof of Proposition \ref{prop:motivator_scale}]
We adapt key elements of the scale control used by \citet[\S{2}]{sun2021arxiv} to our setting. We start by looking at first-order conditions for optimality of $b > 0$. First, note that
\begin{align*}
\frac{\partial}{\partial{b}} \crit_{\ddist}(h;a,b) & = \beta + \lambda\frac{\partial}{\partial{b}} \left( \exx_{\ddist}\sqrt{(\loss(h)-a)^{2}+b^{2}} - b \right)\\
& = \beta + \lambda\exx_{\ddist}\left(\frac{b}{\sqrt{(\loss(h)-a)^{2}+b^{2}}}\right) - \lambda.
\end{align*}
As such, it follows that
\begin{align}\label{eqn:motivator_scale_condition}
\exx_{\ddist}\left(\frac{b}{\sqrt{(\loss(h)-a)^{2}+b^{2}}}\right) = 1 - \beta/\lambda
\end{align}
is equivalent to $\partial_{b}\crit_{\ddist}(h;a,b) = 0$. Obviously, the left-hand side of (\ref{eqn:motivator_scale_condition}) is non-negative for all $b \geq 0$ and bounded above by $1$ for all $b \geq 0$, $a \in \RR$, and $h \in \HH$. Thus (\ref{eqn:motivator_scale_condition}) can only hold for $0 \leq \beta \leq \lambda$.  Using convexity (Lemma \ref{lem:helper_convex_recsq1}) and Jensen's inequality \citep[Thm.~6.3.5]{ash2000a}, we have
\begin{align*}
\exx_{\ddist}\left(\frac{b}{\sqrt{(\loss(h)-a)^{2}+b^{2}}}\right) = \exx_{\ddist}\left(\frac{1}{\sqrt{(\frac{\loss(h)-a}{b})^{2}+1}}\right) \geq \left(\frac{1}{\sqrt{\exx_{\ddist}(\frac{\loss(h)-a}{b})^{2}+1}}\right)
\end{align*}
and thus whenever (\ref{eqn:motivator_scale_condition}) holds, we know that
\begin{align*}
(1-\beta/\lambda)^{2} \geq \frac{1}{\exx_{\ddist}(\frac{\loss(h)-a}{b})^{2}+1}
\end{align*}
must also hold. Re-arranging terms, we see that this implies
\begin{align*}
b^{2} \leq \frac{(1-\beta/\lambda)^{2}\exx_{\ddist}(\loss(h)-a)^{2}}{1-(1-\beta/\lambda)^{2}}.
\end{align*}
For readability, set $\eta \defeq \beta / \lambda$, and note that since
\begin{align*}
\frac{(1-\eta)^{2}}{1-(1-\eta)^{2}} = \frac{(1-\eta)^{2}}{2\eta-\eta^{2}} = \frac{(1-\eta)^{2}}{2\eta(1-\eta/2)} \leq \frac{(1-\eta)^{2}}{2\eta(1-\eta)} \leq \frac{1}{2\eta}
\end{align*}
we can obtain the cleaner (but looser) upper bound
\begin{align*}
b^{2} \leq \frac{\exx_{\ddist}(\loss(h)-a)^{2}}{2\eta} = \frac{\lambda}{2\beta}\exx_{\ddist}(\loss(h)-a)^{2}
\end{align*}
for any choice of $0 < \beta \leq \lambda$ and $a \in \RR$. Since the first-order condition (\ref{eqn:motivator_scale_condition}) is necessary for optimality \citep[Thm.~1.2.1]{nesterov2004ConvOpt}, it follows that
\begin{align}\label{eqn:motivator_scale_upbd}
b_{\ddist}^{2}(h,a) \leq \frac{\lambda}{2\beta} \exx_{\ddist}(\loss(h)-a)^{2}
\end{align}
which is the desired upper bound.

Considering a lower bound next, note first that using the concavity of $x \mapsto \sqrt{x}$ on $\RR_{+}$, another application of Jensen's inequality gives us
\begin{align}\label{eqn:motivator_scale_lowbd_initial}
\exx_{\ddist}\left(\frac{b}{\sqrt{(\loss(h)-a)^{2}+b^{2}}}\right) = \exx_{\ddist}\sqrt{\frac{b^{2}}{(\loss(h)-a)^{2}+b^{2}}} \leq \sqrt{\exx_{\ddist}\left(\frac{1}{(\frac{\loss(h)-a}{b})^{2}+1}\right)}.
\end{align}
Using the inequality $1/(x+1) \leq 1 - x/2$ for all $0 \leq x \leq 1$ ((\ref{eqn:helper_ineq_rec1x}) in Lemma \ref{lem:helper_ineq}), this suggests a natural event to use as a condition. More precisely, writing $\rdv{E} \defeq \indic\left\{\abs{\loss(h)-a} \leq b\right\}$ for readability, note that we have
\begin{align*}
\frac{1}{(\frac{\loss(h)-a}{b})^{2}+1} & = \frac{1-\rdv{E}}{(\frac{\loss(h)-a}{b})^{2}+1} + \frac{\rdv{E}}{(\frac{\loss(h)-a}{b})^{2}+1}\\
& \leq \frac{1-\rdv{E}}{(\frac{\loss(h)-a}{b})^{2}+1} + \rdv{E}\left(1-\frac{1}{2}\left(\frac{\loss(h)-a}{b}\right)^{2}\right)\\
& = \underbrace{\left(\frac{1-\rdv{E}}{(\frac{\loss(h)-a}{b})^{2}+1} + \rdv{E}\right)}_{\leq 1} - \frac{\rdv{E}}{2}\left(\frac{\loss(h)-a}{b}\right)^{2}.
\end{align*}
Taking expectation and utilizing (\ref{eqn:motivator_scale_lowbd_initial}), whenever (\ref{eqn:motivator_scale_condition}) holds, we have
\begin{align}\label{eqn:motivator_scale_lowbd_useful}
1-\beta/\lambda = \exx_{\ddist}\left(\frac{b}{\sqrt{(\loss(h)-a)^{2}+b^{2}}}\right) \leq \sqrt{1 - \exx_{\ddist}\frac{\rdv{E}}{2}\left(\frac{\loss(h)-a}{b}\right)^{2}}.
\end{align}
With this established, note that via helper inequality (\ref{eqn:helper_ineq_sun}), for any $\beta \leq \lambda$ we have
\begin{align*}
(1-\beta/\lambda)^{2} \geq 1 - 2\beta/\lambda
\end{align*}
and thus in light of (\ref{eqn:motivator_scale_lowbd_useful}), we may conclude that
\begin{align*}
1 - 2\beta/\lambda \leq 1 - \exx_{\ddist}\frac{\rdv{E}}{2}\left(\frac{\loss(h)-a}{b}\right)^{2}
\end{align*}
which implies
\begin{align*}
\frac{\lambda}{4\beta}\exx_{\ddist}\rdv{E}(h;a,b)(\loss(h)-a)^{2} \leq b^{2}
\end{align*}
noting that we have written $\rdv{E}(h;a,b)$ to emphasize the dependence on $h$, $a$, and $b$. Once again since the first-order condition (\ref{eqn:motivator_scale_condition}) is necessary for optimality, we may conclude that
\begin{align}\label{eqn:motivator_scale_lowbd}
\frac{\lambda}{4\beta}\exx_{\ddist}\rdv{E}(h;a,b_{\ddist}(h,a))(\loss(h)-a)^{2} \leq b_{\ddist}^{2}(h,a)
\end{align}
which is the remaining desired inequality.
\end{proof}

\bigskip

\begin{proof}[Proof of the limit (\ref{eqn:beta_makes_scale_grow})]
Recall from the proof of Proposition \ref{prop:motivator_scale} the first-order optimality condition (\ref{eqn:motivator_scale_condition}), which is satisfied by any solution $b_{\ddist}(h,a)$ given by (\ref{eqn:optimal_scale}), i.e., we have
\begin{align}\label{eqn:beta_makes_scale_grow_1}
\exx_{\ddist}\left(\frac{b_{\ddist}(h,a)}{\sqrt{(\loss(h)-a)^{2}+(b_{\ddist}(h,a))^{2}}}\right) = 1 - \beta/\lambda
\end{align}
for any $0 < \beta \leq \lambda$. Defining $g(\beta) \defeq 1-\beta/\lambda$ and taking any $0 < \beta_{2} < \beta_{1} \leq \lambda$, clearly we have $g(\beta_{1}) < g(\beta_{2})$ and thus using the equality (\ref{eqn:beta_makes_scale_grow_1}), we must have that $b_{\ddist}(h,a;\beta_{2}) \geq b_{\ddist}(h,a;\beta_{1})$, otherwise it would result in a contradiction of (\ref{eqn:beta_makes_scale_grow_1}). Using this monotonicity, clearly
\begin{align*}
\rdv{E}(h;a,b_{\ddist}(h,a;\beta_{1})) \leq \rdv{E}(h;a,b_{\ddist}(h,a;\beta_{2}))
\end{align*}
and thus
\begin{align*}
\exx_{\ddist}\rdv{E}(h;a,b_{\ddist}(h,a;\beta_{1}))(\loss(h)-a)^{2} \leq \rdv{E}(h;a,b_{\ddist}(h,a;\beta_{2}))(\loss(h)-a)^{2}.
\end{align*}
Applying this to the lower bound in Proposition \ref{prop:motivator_scale}, we have
\begin{align*}
\liminf_{\beta \to 0_{+}} b_{\ddist}^{2}(h,a) \geq \lim_{\beta \to 0_{+}} \frac{\lambda}{4\beta}\exx_{\ddist}\rdv{I}(h;a)(\loss(h)-a)^{2} = \infty
\end{align*}
as desired.
\end{proof}

\bigskip

\begin{proof}[Proof of the limit (\ref{eqn:robust_deviations_vanish})]
Note that we can easily bound the random variable of interest as
\begin{align}\label{eqn:robust_deviations_vanish_1}
0 \leq b\rho\left(\frac{\loss(h)-a}{b}\right) = \sqrt{(\loss(h)-a)^{2}+b^{2}} - b \leq \abs{\loss(h)-a}
\end{align}
for any choice of $0 < b < \infty$. Some straightforward calculus shows that
\begin{align*}
\lim\limits_{b \to \infty} b\rho\left(\frac{\loss(h)-a}{b}\right) = 0
\end{align*}
in a pointwise sense. Since the upper bound in (\ref{eqn:robust_deviations_vanish_1}) is $\ddist$-integrable by assumption, a simple application of dominated convergence \citep[Thm.~1.6.9]{ash2000a} yields
\begin{align*}
\lim\limits_{b \to \infty} b \exx_{\ddist}\rho\left(\frac{\loss(h)-a}{b}\right) = \exx_{\ddist} \left[ \lim\limits_{b \to \infty} b\rho\left(\frac{\loss(h)-a}{b}\right) \right] = 0
\end{align*}
as desired.
\end{proof}

\bigskip

\begin{proof}[Proof of Proposition \ref{prop:location_concentration}]
From condition (\ref{eqn:optimal_empirical_location}), since any solution must also be a stationary point \citep[Thm.~1.2.1]{nesterov2004ConvOpt}, we know that $\rdv{A}_{n} \defeq \rdv{A}_{n}(h,b)$ must satisfy the first-order condition 
\begin{align*}
\frac{\lambda}{n} \sum_{i=1}^{n} \rho^{\prime}\left(\frac{\loss_{i}(h)-\rdv{A}_{n}}{b}\right) = \alpha
\end{align*}
which is equivalent to 
\begin{align}\label{eqn:location_concentration_0}
\frac{b}{n} \sum_{i=1}^{n} \rho^{\prime}\left(\frac{\loss_{i}(h)-\rdv{A}_{n}}{b}\right) = \frac{\alpha}{\lambda}b.
\end{align}
Next we make use of the argument developed by \citet[\S{2}]{catoni2012a}. First note that fixing any $a \in \RR$ and $b > 0$, we have
\begin{align}
\nonumber
\exx\exp\left(\sum_{i=1}^{n}\rho^{\prime}\left(\frac{\loss_{i}(h)-a}{b}\right)\right) & = \exx\left[\prod_{i=1}^{n}\exp\left(\rho^{\prime}\left(\frac{\loss_{i}(h)-a}{b}\right)\right)\right]\\
\nonumber
& = \prod_{i=1}^{n}\exx_{i}\exp\left(\rho^{\prime}\left(\frac{\loss_{i}(h)-a}{b}\right)\right)\\
\nonumber
& \leq \prod_{i=1}^{n} \exx_{i}\left(1 + \frac{\loss_{i}(h)-a}{b} + \frac{\gamma}{b^{2}}(\loss_{i}(h)-a)^{2}\right)\\
\nonumber
& = \left(1 + \frac{\exx_{\ddist}\loss(h)-a}{b} + \frac{\gamma}{b^{2}}\exx_{\ddist}(\loss(h)-a)^{2}\right)^{n}\\
\label{eqn:location_concentration_1}
& \leq \exp\left( \frac{n}{b}\left(\exx_{\ddist}\loss(h)-a\right) + \frac{n\gamma}{b^{2}}\exx_{\ddist}(\loss(h)-a)^{2} \right).
\end{align}
The second equality above follows from the independence of the training data, and the first inequality uses the upper bound in (\ref{eqn:motivation_catoni_condition}), which is satisfied by $\rho$ given in (\ref{eqn:sun_huber_fn}) with $\gamma = 1$, though we leave $\gamma$ as is to illustrate how more general results are obtained. The third equality just uses the fact that the training data is an IID sample from $\ddist$, and the final inequality culminating in (\ref{eqn:location_concentration_1}) just uses the bound $1+x \leq \exp(x)$. Using Markov's inequality and taking $0 < \delta < 1$, it is straightforward to show that (\ref{eqn:location_concentration_1}) implies a $1-\delta$ event (over the draw of $\rdv{Z}_{1},\ldots,\rdv{Z}_{n}$) in which we have
\begin{align*}
\sum_{i=1}^{n}\rho^{\prime}\left(\frac{\loss_{i}(h)-a}{b}\right) \leq \frac{n}{b}\left(\exx_{\ddist}\loss(h)-a\right) + \frac{n\gamma}{b^{2}}\exx_{\ddist}(\loss(h)-a)^{2} + \log(1/\delta).
\end{align*}
Multiplying both sides by $b/n$, on the same ``good'' event, we have
\begin{align}
\nonumber
\frac{b}{n}\sum_{i=1}^{n}\rho^{\prime}\left(\frac{\loss_{i}(h)-a}{b}\right) & \leq \exx_{\ddist}\loss(h)-a + \frac{\gamma}{b}\exx_{\ddist}(\loss(h)-a)^{2} + \frac{b\log(1/\delta)}{n}\\
\label{eqn:location_concentration_2}
& = \exx_{\ddist}\loss(h)-a + \frac{\gamma}{b}\left(\vaa_{\ddist}\loss(h) + (\exx_{\ddist}\loss(h)-a)^{2}\right) + \frac{b\log(1/\delta)}{n}
\end{align}
where (\ref{eqn:location_concentration_2}) follows from expanding the quadratic term and doing some algebra. With the equality (\ref{eqn:location_concentration_0}) in mind, subtracting a constant from both sides of (\ref{eqn:location_concentration_2}), note that we equivalently have
\begin{align}\label{eqn:location_concentration_3}
\frac{b}{n}\sum_{i=1}^{n}\rho^{\prime}\left(\frac{\loss_{i}(h)-a}{b}\right) - \frac{\alpha}{\lambda}b \leq p(a)
\end{align}
where we have defined
\begin{align}\label{eqn:location_concentration_4}
p(a) \defeq \exx_{\ddist}\loss(h)-a + \frac{\gamma}{b}\left(\vaa_{\ddist}\loss(h) + (\exx_{\ddist}\loss(h)-a)^{2}\right) + \frac{b\log(1/\delta)}{n} - \frac{\alpha}{\lambda}b
\end{align}
for readability. Note that $p(\cdot)$ in (\ref{eqn:location_concentration_4}) is a polynomial of degree 2, and can be written as
\begin{align}
p(a) = ua^{2} + va + w
\end{align}
with coefficients defined as
\begin{align*}
u & \defeq \frac{\gamma}{b}\\
v & \defeq (-1)\left(1 + \frac{2\gamma\exx_{\ddist}\loss(h)}{b}\right)\\
w & \defeq \exx_{\ddist}\loss(h) + \frac{\gamma}{b} \exx_{\ddist}\abs{\loss(h)}^{2} + \frac{b\log(1/\delta)}{n} - \frac{\alpha}{\lambda}b.
\end{align*}
This polynomial has real roots whenever $v^{2} - 4uw \geq 0$, and some algebra shows that this is equivalent to
\begin{align}\label{eqn:location_concentration_polycondition_up}
0 \leq D \leq 1, \text{ where } D \defeq 4\left( \left(\frac{\gamma}{b}\right)^{2}\vaa_{\ddist}\loss(h) + \frac{\gamma\log(1/\delta)}{n} - \frac{\gamma\alpha}{\lambda} \right).
\end{align}
Assuming this holds, denoting by $a_{+}$ the smallest root of $p(\cdot)$, i.e., the smallest of satisfying $p(a_{+})=0$, the critical fact of interest to us is that $\rdv{A}_{n} \leq a_{+}$ on the good event of (\ref{eqn:location_concentration_3}). This is valid due to two facts: first, the left-hand side of (\ref{eqn:location_concentration_3}) is a decreasing function of $a$; second, due to (\ref{eqn:location_concentration_0}), we know that $\rdv{A}_{n}$ is a root of the left-hand side of (\ref{eqn:location_concentration_3}). With this key fact in hand, using the quadratic formula we have
\begin{align*}
\rdv{A}_{n} & \leq a_{+}\\
& = \exx_{\ddist}\loss(h) + \frac{b}{2\gamma}\left(1 - \sqrt{1-D}\right)\\
& = \exx_{\ddist}\loss(h) + \frac{b}{2\gamma}\frac{\left(1 - \sqrt{1-D}\right)\left(1 + \sqrt{1-D}\right)}{\left(1 + \sqrt{1-D}\right)}\\
& = \exx_{\ddist}\loss(h) + \frac{b}{2\gamma}\frac{D}{\left(1 + \sqrt{1-D}\right)}\\
& \leq \exx_{\ddist}\loss(h) + \frac{b}{2\gamma}D.
\end{align*}
Taking the two ends of this inequality chain together and expanding $D$, we have
\begin{align}\label{eqn:location_concentration_upbd}
\rdv{A}_{n} \leq \exx_{\ddist}\loss(h) - 2(\alpha/\lambda)b + 2\left( \frac{\gamma}{b}\vaa_{\ddist}\loss(h) + \frac{b\log(1/\delta)}{n} \right)
\end{align}
with probability no less than $1-\delta$, assuming that $n$, $b$, and $\alpha$ are such that $0 \leq D \leq 1$ holds. This gives us the desired upper bound.

To obtain a lower bound, a perfectly analogous argument can be applied. First, using the \emph{lower} bound in (\ref{eqn:motivation_catoni_condition}) and the fact that $\rho^{\prime}(-x) = -\rho^{\prime}(x)$, we know that
\begin{align}\label{eqn:location_concentration_5}
\rho^{\prime}\left(\frac{a - \loss_{i}(h)}{b}\right) \leq \log\left(1 + \frac{a - \loss_{i}(h)}{b} + \frac{\gamma}{b^{2}}(a - \loss_{i}(h))^{2}\right)
\end{align}
for any $a \in \RR$, $b > 0$, and $i \in [n]$. Plugging this inequality (\ref{eqn:location_concentration_5}) into an argument analogous to the chain of inequalities that led to (\ref{eqn:location_concentration_2}) earlier, it is clear that again on an event of probability no less than $1-\delta$, we have
\begin{align}\label{eqn:location_concentration_6}
\frac{b}{n}\sum_{i=1}^{n}\rho^{\prime}\left(\frac{a-\loss_{i}(h)}{b}\right) \leq a-\exx_{\ddist}\loss(h) + \frac{\gamma}{b}\left(\vaa_{\ddist}\loss(h) + (\exx_{\ddist}\loss(h)-a)^{2}\right) + \frac{b\log(1/\delta)}{n}.
\end{align}
Once again the upper bound we can bound this using a polynomial of degree 2, namely
\begin{align}\label{eqn:location_concentration_7}
\frac{b}{n}\sum_{i=1}^{n}\rho^{\prime}\left(\frac{a-\loss_{i}(h)}{b}\right) + \frac{\alpha}{\lambda}b \leq q(a)
\end{align}
where we have defined
\begin{align}
q(a) \defeq a-\exx_{\ddist}\loss(h) + \frac{\gamma}{b}\left(\vaa_{\ddist}\loss(h) + (\exx_{\ddist}\loss(h)-a)^{2}\right) + \frac{b\log(1/\delta)}{n} + \frac{\alpha}{\lambda}b.
\end{align}
Now, since $\rdv{A}_{n}$ is a root of the left-hand side of (\ref{eqn:location_concentration_7}) viewed as a function of $a$, and this function is monotonically increasing, it is evident that denoting the largest root of $q(\cdot)$ (when it exists) by $a_{-}$, we have $\rdv{A}_{n} \geq a_{-}$, a lower bound in contrast to the $\rdv{A}_{n} \leq a_{+}$ upper bound used earlier. For completeness, we write this polynomial as
\begin{align}
q(a) = u^{\prime}a^{2} + v^{\prime}a + w^{\prime}
\end{align}
with coefficients
\begin{align*}
u^{\prime} & \defeq \frac{\gamma}{b}\\
v^{\prime} & \defeq \left(1 - \frac{2\gamma\exx_{\ddist}\loss(h)}{b}\right)\\
w^{\prime} & \defeq (-1)\exx_{\ddist}\loss(h) + \frac{\gamma}{b} \exx_{\ddist}\abs{\loss(h)}^{2} + \frac{b\log(1/\delta)}{n} + \frac{\alpha}{\lambda}b.
\end{align*}
We have two real roots whenever
\begin{align}\label{eqn:location_concentration_polycondition_lo}
1 \geq D^{\prime} \defeq 4\left( \left(\frac{\gamma}{b}\right)^{2}\vaa_{\ddist}\loss(h) + \frac{\gamma\log(1/\delta)}{n} + \frac{\gamma\alpha}{\lambda} \right)
\end{align}
holds, and thus we obtain a high probability lower bound on $\rdv{A}_{n}$ as follows:
\begin{align*}
\rdv{A}_{n} & \geq a_{-}\\
& = \exx_{\ddist}\loss(h) - \frac{b}{2\gamma}\left(1-\sqrt{1-D^{\prime}}\right)\\
& \geq \exx_{\ddist}\loss(h) - \frac{b}{2\gamma}D^{\prime}.
\end{align*}
Expanding $D^{\prime}$ gives us the lower bound
\begin{align}\label{eqn:location_concentration_lobd}
\rdv{A}_{n} \geq \exx_{\ddist}\loss(h) - 2(\alpha/\lambda)b - 2\left( \frac{\gamma}{b}\vaa_{\ddist}\loss(h) + \frac{b\log(1/\delta)}{n} \right)
\end{align}
with probability no less than $1-\delta$, as desired.

Let us conclude this proof by organizing the technical assumptions. First of all, for the two quadratics used in the preceding bounds, we require both (\ref{eqn:location_concentration_polycondition_up}) and (\ref{eqn:location_concentration_polycondition_lo}) to hold. It is straightforward to verify that having these conditions both hold is equivalent to the following:
\begin{align}\label{eqn:location_concentration_polycondition_joint}
\frac{4\gamma\alpha}{\lambda} \leq 4\left( \left(\frac{\gamma}{b}\right)^{2}\vaa_{\ddist}\loss(h) + \frac{\gamma\log(1/\delta)}{n} \right) \leq 1 - \frac{4\gamma\alpha}{\lambda}.
\end{align}
As such, whenever $\alpha$, $\delta$, and $b$ are such that (\ref{eqn:location_concentration_polycondition_joint}) holds, using a union bound, it follows that with probability no less than $1-2\delta$, we have a bound on
\begin{align*}
\abs{\rdv{A}_{n} - (\exx_{\ddist}\loss(h)-2(\alpha/\lambda)b)} \leq 2\left( \frac{\gamma}{b}\vaa_{\ddist}\loss(h) + \frac{b\log(1/\delta)}{n} \right)
\end{align*}
as desired. The proposition statement takes a cleaner form since we have $\gamma=1$.
\end{proof}

\bigskip

\begin{proof}[Proof of Proposition \ref{prop:nonsmooth}]
The lack of convexity follows from the fact that the composition of two convex functions need not be convex when the outermost function is \emph{non-monotonic} (see for example \citet[Ch.~3]{boyd2004ConvOpt}), and the lack of smoothness follows \textit{a fortiori} from Lemma \ref{lem:partial_objective}.
\end{proof}

\end{document}

%% file: config/bib.tex
\usepackage[numbers]{natbib} 

%% file: config/mathfonts.tex
\usepackage{amsmath} 
\usepackage{amssymb} 
\usepackage{amsfonts} 
\usepackage{bm} 
\usepackage{amsthm} 
\theoremstyle{definition} \newtheorem{defn}{Definition}
\theoremstyle{plain} \newtheorem{prop}[defn]{Proposition}
\theoremstyle{plain} 
\theoremstyle{plain} \newtheorem{lem}[defn]{Lemma}
\theoremstyle{plain} 
\theoremstyle{remark} 
\theoremstyle{remark}

%% file: config/symbols.tex
\newcommand{\term}[1]{\textcolor{BlueViolet}{\textit{{#1}}}}

\usepackage{enumitem}
\makeatletter
\def\namedlabel#1#2{\begingroup
    #2%
    \def\@currentlabel{#2}%
    \phantomsection\label{#1}\endgroup
}
\makeatother
\newcommand*{\defeq}{\mathrel{\vcenter{\baselineskip0.5ex \lineskiplimit0pt     
                     \hbox{\scriptsize.}\hbox{\scriptsize.}}}=}

\newcommand{\overbar}[1]{\mkern 1.5mu\overline{\mkern-1.5mu#1\mkern-1.5mu}\mkern 1.5mu}

\DeclareMathOperator*{\argmin}{arg\,min}

\newcommand{\Abs}[1]{\lVert{#1}\rVert} 
\newcommand{\abs}[1]{\lvert{#1}\rvert} 
\def\bigO{\mathcal{O}} 
\DeclareMathOperator{\crit}{C} 
\def\ddist{\mu} 
\DeclareMathOperator{\exx}{\mathbf{E}} 
\def\HH{\mathcal{H}} 
\DeclareMathOperator{\indic}{I} 
\DeclareMathOperator{\loss}{\rdv{L}} 
\DeclareMathOperator{\msd}{MS} 
\DeclareMathOperator{\mv}{MV} 
\newcommand{\rdv}[1]{\mathsf{#1}} 
\def\RR{\mathbb{R}} 
\DeclareMathOperator{\sign}{sign} 
\DeclareMathOperator{\vaa}{\mathbf{V}} 

%% file: bdd-mv_arxiv.bbl
\begin{thebibliography}{}

\bibitem[Ash and Dol{\'e}ans-Dade, 2000]{ash2000a}
Ash, R.~B. and Dol{\'e}ans-Dade, C.~A. (2000).
\newblock {\em Probability and Measure Theory}.
\newblock Academic Press, 2nd edition.

\bibitem[Barron, 2019]{barron2019a}
Barron, J.~T. (2019).
\newblock A general and adaptive robust loss function.
\newblock In {\em Proceedings of the IEEE/CVF Conference on Computer Vision and
  Pattern Recognition}, pages 4331--4339.

\bibitem[Ben-Tal et~al., 2013]{bental2013a}
Ben-Tal, A., Den~Hertog, D., De~Waegenaere, A., Melenberg, B., and Rennen, G.
  (2013).
\newblock Robust solutions of optimization problems affected by uncertain
  probabilities.
\newblock {\em Management Science}, 59(2):341--357.

\bibitem[Boyd and Vandenberghe, 2004]{boyd2004ConvOpt}
Boyd, S. and Vandenberghe, L. (2004).
\newblock {\em Convex Optimization}.
\newblock Cambridge University Press.

\bibitem[Brownlees et~al., 2015]{brownlees2015a}
Brownlees, C., Joly, E., and Lugosi, G. (2015).
\newblock Empirical risk minimization for heavy-tailed losses.
\newblock {\em The Annals of Statistics}, 43(6):2507--2536.

\bibitem[Catoni, 2012]{catoni2012a}
Catoni, O. (2012).
\newblock Challenging the empirical mean and empirical variance: a deviation
  study.
\newblock {\em Annales de l'Institut Henri Poincar{\'e}, Probabilit{\'e}s et
  Statistiques}, 48(4):1148--1185.

\bibitem[Curi et~al., 2020]{curi2020a}
Curi, S., Levy, K.~Y., Jegelka, S., and Krause, A. (2020).
\newblock Adaptive sampling for stochastic risk-averse learning.
\newblock In {\em Advances in Neural Information Processing Systems 33 (NeurIPS
  2020)}, pages 1036--1047.

\bibitem[Davis and Drusvyatskiy, 2019]{davis2019b}
Davis, D. and Drusvyatskiy, D. (2019).
\newblock Stochastic model-based minimization of weakly convex functions.
\newblock {\em SIAM Journal on Optimization}, 29(1):207--239.

\bibitem[Devroye et~al., 1996]{devroye1996ProbPR}
Devroye, L., Gy{\"o}rfi, L., and Lugosi, G. (1996).
\newblock {\em A Probabilistic Theory of Pattern Recognition}.
\newblock Springer.

\bibitem[Devroye et~al., 2016]{devroye2016a}
Devroye, L., Lerasle, M., Lugosi, G., and Oliveira, R.~I. (2016).
\newblock Sub-{G}aussian mean estimators.
\newblock {\em The Annals of Statistics}, 44(6):2695--2725.

\bibitem[Duchi and Namkoong, 2019]{duchi2019a}
Duchi, J. and Namkoong, H. (2019).
\newblock Variance-based regularization with convex objectives.
\newblock {\em Journal of Machine Learning Research}, 20(68):1--55.

\bibitem[Ghadimi and Lan, 2013]{ghadimi2013a}
Ghadimi, S. and Lan, G. (2013).
\newblock Stochastic first- and zeroth-order methods for nonconvex stochastic
  programming.
\newblock {\em SIAM Journal on Optimization}, 23(4):2341--2368.

\bibitem[Hashimoto et~al., 2018]{hashimoto2018a}
Hashimoto, T.~B., Srivastava, M., Namkoong, H., and Liang, P. (2018).
\newblock Fairness without demographics in repeated loss minimization.
\newblock In {\em Proceedings of the 35th International Conference on Machine
  Learning (ICML)}, volume~80 of {\em Proceedings of Machine Learning
  Research}, pages 1929--1938.

\bibitem[Holland and Tanabe, 2023]{holland2023survey}
Holland, M.~J. and Tanabe, K. (2023).
\newblock A survey of learning criteria going beyond the usual risk.
\newblock {\em Journal of Artificial Intelligence Research}, 73:781--821.

\bibitem[Hsu and Sabato, 2016]{hsu2016a}
Hsu, D. and Sabato, S. (2016).
\newblock Loss minimization and parameter estimation with heavy tails.
\newblock {\em Journal of Machine Learning Research}, 17(18):1--40.

\bibitem[Hu et~al., 2022]{hu2022a}
Hu, S., Wang, X., and Lyu, S. (2022).
\newblock Rank-based decomposable losses in machine learning: A survey.
\newblock {\em arXiv preprint arXiv:2207.08768v1}.

\bibitem[Huber, 1964]{huber1964a}
Huber, P.~J. (1964).
\newblock Robust estimation of a location parameter.
\newblock {\em The Annals of Mathematical Statistics}, 35(1):73--101.

\bibitem[Huber and Ronchetti, 2009]{huber2009a}
Huber, P.~J. and Ronchetti, E.~M. (2009).
\newblock {\em Robust Statistics}.
\newblock John Wiley \& Sons, 2nd edition.

\bibitem[Lee et~al., 2020]{lee2020a}
Lee, J., Park, S., and Shin, J. (2020).
\newblock Learning bounds for risk-sensitive learning.
\newblock In {\em Advances in Neural Information Processing Systems 33 (NeurIPS
  2020)}, pages 13867--13879.

\bibitem[Li et~al., 2021]{li2021a}
Li, T., Beirami, A., Sanjabi, M., and Smith, V. (2021).
\newblock Tilted empirical risk minimization.
\newblock In {\em The 9th International Conference on Learning Representations
  (ICLR)}.

\bibitem[Luenberger, 1969]{luenberger1969Book}
Luenberger, D.~G. (1969).
\newblock {\em Optimization by Vector Space Methods}.
\newblock John Wiley \& Sons.

\bibitem[Lugosi and Mendelson, 2019]{lugosi2019b}
Lugosi, G. and Mendelson, S. (2019).
\newblock Mean estimation and regression under heavy-tailed distributions: A
  survey.
\newblock {\em Foundations of Computational Mathematics}, 19(5):1145--1190.

\bibitem[Markowitz, 1952]{markowitz1952a}
Markowitz, H. (1952).
\newblock Portfolio selection.
\newblock {\em Journal of Finance}, 7(1):77--91.

\bibitem[Maurer and Pontil, 2009]{maurer2009a}
Maurer, A. and Pontil, M. (2009).
\newblock Empirical {B}ernstein bounds and sample variance penalization.
\newblock In {\em Proceedings of the 22nd Conference on Learning Theory
  (COLT)}.

\bibitem[Medina and Yang, 2021]{maurer2021a}
Medina, A.~M. and Yang, S. (2021).
\newblock Robust unsupervised learning via {L}-statistic minimization.
\newblock In {\em 38th International Conference on Machine Learning (ICML)},
  volume 139 of {\em Proceedings of Machine Learning Research}, pages
  7524--7533.

\bibitem[Menon et~al., 2021]{menon2021a}
Menon, A.~K., Jayasumana, S., Rawat, A.~S., Jain, H., Veit, A., and Kumar, S.
  (2021).
\newblock Long-tail learning via logit adjustment.
\newblock In {\em The 9th International Conference on Learning Representations
  (ICLR)}.

\bibitem[Mohri et~al., 2012]{mohri2012Foundations}
Mohri, M., Rostamizadeh, A., and Talwalkar, A. (2012).
\newblock {\em Foundations of Machine Learning}.
\newblock MIT Press.

\bibitem[Nesterov, 2004]{nesterov2004ConvOpt}
Nesterov, Y. (2004).
\newblock {\em Introductory Lectures on Convex Optimization: A Basic Course}.
\newblock Springer.

\bibitem[Rey, 1983]{rey1983Robust}
Rey, W. J.~J. (1983).
\newblock {\em Introduction to Robust and Quasi-Robust Statistical Methods}.
\newblock Springer.

\bibitem[Rockafellar and Uryasev, 2000]{rockafellar2000a}
Rockafellar, R.~T. and Uryasev, S. (2000).
\newblock Optimization of conditional value-at-risk.
\newblock {\em Journal of Risk}, 2:21--42.

\bibitem[Rockafellar and Uryasev, 2013]{rockafellar2013a}
Rockafellar, R.~T. and Uryasev, S. (2013).
\newblock The fundamental risk quadrangle in risk management, optimization and
  statistical estimation.
\newblock {\em Surveys in Operations Research and Management Science},
  18(1-2):33--53.

\bibitem[Royset, 2022]{royset2022arxiv}
Royset, J.~O. (2022).
\newblock Risk-adaptive approaches to learning and decision making: A survey.
\newblock {\em arXiv preprint arXiv:2212.00856}.

\bibitem[Shapiro, 2017]{shapiro2017a}
Shapiro, A. (2017).
\newblock Distributionally robust stochastic programming.
\newblock {\em SIAM Journal on Optimization}, 27(4):2258--2275.

\bibitem[Sun, 2021]{sun2021arxiv}
Sun, Q. (2021).
\newblock Do we need to estimate the variance in robust mean estimation?
\newblock {\em arXiv preprint arXiv:2107.00118v1}.

\bibitem[Vapnik, 1999]{vapnik1999NSLT}
Vapnik, V.~N. (1999).
\newblock {\em The Nature of Statistical Learning Theory}.
\newblock Statistics for Engineering and Information Science. Springer, 2nd
  edition.

\bibitem[Zhai et~al., 2021]{zhai2021a}
Zhai, R., Dan, C., Kolter, J.~Z., and Ravikumar, P. (2021).
\newblock {DORO}: Distributional and outlier robust optimization.
\newblock In {\em 38th International Conference on Machine Learning (ICML)},
  volume 139 of {\em Proceedings of Machine Learning Research}, pages
  12345--12355.

\end{thebibliography}
